\documentclass{article}

\usepackage[T1]{fontenc}
\usepackage{tikz}
\usetikzlibrary{arrows}
\usetikzlibrary{calc}
\usetikzlibrary{snakes}
\usepackage{cite}
\usepackage{graphicx}
\usepackage{array}
\usepackage{mdwmath}
\usepackage{mdwtab}
\usepackage{eqparbox}
\usepackage{sidecap}
\usepackage{fixltx2e}
\usepackage{url}
\usepackage{pgf}
\usepackage{lipsum}
\usepackage{multirow}
\usepackage{subcaption}
\usepackage[cmex10]{amsmath}
\usepackage{amsfonts,amssymb,amsthm}
\usepackage{xspace}
\usepackage{xcolor}
\usepackage{algorithmic}
\usepackage{algorithm}
\usepackage{bm}
\usepackage{bbm}
\usepackage{booktabs}
\usepackage{etoolbox}
\usepackage{makecell}
\usepackage{nicefrac}
\usepackage{textcomp}
\usepackage{stmaryrd}
\usepackage{mathrsfs}
\usepackage{paralist}
\usepackage{comment}

\usepackage[font=footnotesize,labelsep=space,labelfont=bf]{caption}

\usepackage{url}

\usepackage[pdftex,bookmarks=true,pdfstartview=FitH,colorlinks,linkcolor=blue,filecolor=blue,citecolor=blue,urlcolor=blue,pagebackref=false]{hyperref}
\makeatletter

\newcommand{\Rmnum}[1]{\expandafter\@slowromancap\romannumeral #1@}
\makeatother

\newtheorem{theorem}{Theorem}
\newtheorem*{theorem*}{Theorem}
\newtheorem{lemma}{Lemma}
\newtheorem*{lemma*}{Lemma}
\newtheorem{claim}{Claim}
\newtheorem*{claim*}{Claim}
\newtheorem*{cor*}{Corollary}
\newtheorem{remark}{Remark}
\newtheorem{fact}{Fact}
\newtheorem{assumption}{Assumption}

\newcommand{\namedref}[2]{\hyperref[#2]{#1~\ref*{#2}}}

\newcommand{\GD}{\text{GD}}

\newcommand{\accu}{\text{accu}}

\newcommand{\Algorithmref}[1]{\namedref{Algorithm}{algo:#1}}
\newcommand{\Assumptionref}[1]{\namedref{Assumption}{assump:#1}}

\newcommand{\Sectionref}[1]{\namedref{Section}{sec:#1}}
\newcommand{\Subsectionref}[1]{\namedref{Section}{subsec:#1}}

\newcommand{\Appendixref}[1]{\namedref{Appendix}{app:#1}}
\newcommand{\Theoremref}[1]{\namedref{Theorem}{thm:#1}}

\newcommand{\Lemmaref}[1]{\namedref{Lemma}{lem:#1}}

\newcommand{\Claimref}[1]{\namedref{Claim}{claim:#1}}
\newcommand{\Factref}[1]{\namedref{Fact}{fact:#1}}

\newcommand{\Footnoteref}[1]{\namedref{Footnote}{foot:#1}}
\newcommand{\Pageref}[1]{\hyperref[#1]{page~\pageref*{#1}}}

\newcommand{\eps}{\ensuremath{\epsilon}\xspace}

\definecolor{darkred}{rgb}{0.5, 0, 0} 
\definecolor{darkblue}{rgb}{0,0,0.5} 
\hypersetup{
    colorlinks=true,
    linkcolor=darkred,
    citecolor=darkblue,
    urlcolor=darkblue   
}

\renewcommand{\(}{\ensuremath{\left(}\xspace}
\renewcommand{\)}{\ensuremath{\right)}\xspace}

\newcommand{\vct}[1]{\boldsymbol{#1}}

\newcommand{\bmu}{\ensuremath{\boldsymbol{\mu}}\xspace}

\newcommand{\bzero}{\ensuremath{{\bf 0}}\xspace}

\newcommand{\ba}{\ensuremath{\boldsymbol{a}}\xspace}
\newcommand{\bfb}{\ensuremath{\boldsymbol{b}}\xspace}

\newcommand{\bg}{\ensuremath{\boldsymbol{g}}\xspace}
\newcommand{\btg}{\ensuremath{\widehat{\bg}}\xspace}
\newcommand{\btlg}{\ensuremath{\widetilde{\bg}}\xspace}

\newcommand{\bx}{\ensuremath{\boldsymbol{x}}\xspace}

\newcommand{\btlx}{\ensuremath{\widetilde{\bx}}\xspace}

\newcommand{\by}{\ensuremath{\boldsymbol{y}}\xspace}

\newcommand{\bw}{\ensuremath{\boldsymbol{w}}\xspace}

\newcommand{\bu}{\ensuremath{\boldsymbol{u}}\xspace}
\newcommand{\bv}{\ensuremath{\boldsymbol{v}}\xspace}

\newcommand{\bA}{\ensuremath{{\bf A}}\xspace}
\newcommand{\bB}{\ensuremath{{\bf B}}\xspace}

\newcommand{\btF}{\ensuremath{\widehat{F}}\xspace}

\newcommand{\bG}{\ensuremath{{\bf G}}\xspace}

\newcommand{\calA}{\ensuremath{\mathcal{A}}\xspace}

\newcommand{\calD}{\ensuremath{\mathcal{D}}\xspace}
\newcommand{\calE}{\ensuremath{\mathcal{E}}\xspace}
\newcommand{\calS}{\ensuremath{\mathcal{S}}\xspace}

\newcommand{\bW}{\ensuremath{{\bf W}}\xspace}

\newcommand{\bY}{\ensuremath{{\bf Y}}\xspace}

\newcommand{\bbE}{\ensuremath{\mathbb{E}}\xspace}
\newcommand{\calF}{\ensuremath{\mathcal{F}}\xspace}
\newcommand{\calH}{\ensuremath{\mathcal{H}}\xspace}
\newcommand{\calK}{\ensuremath{\mathcal{K}}\xspace}
\newcommand{\calO}{\ensuremath{\mathcal{O}}\xspace}
\newcommand{\I}{\ensuremath{\mathcal{I}}\xspace}

\newcommand{\R}{\ensuremath{\mathbb{R}}\xspace}

\renewcommand{\paragraph}[1]{\smallskip\noindent{\bf #1}~}

\usepackage{fullpage}

\usepackage[utf8]{inputenc} 
\usepackage[T1]{fontenc}    
\usepackage{hyperref}       
\usepackage{url}            
\usepackage{booktabs}      
\usepackage{amsfonts}      
\usepackage{nicefrac}       
\usepackage{microtype}    

\title{Byzantine-Resilient High-Dimensional Federated Learning}

\author{Deepesh Data and Suhas Diggavi \\
University of California, Los Angeles\\
 Email: \{\texttt{deepesh.data@gmail.com, suhas@ee.ucla.edu}\} \\
}

\date{}

\begin{document}

\maketitle

\begin{abstract}
We study stochastic gradient descent (SGD) with local iterations in the presence of malicious/Byzantine clients, motivated by the federated learning.  The clients, instead of communicating with the central server in every iteration, maintain their local models, which they update by taking several SGD iterations based on their own datasets and then communicate the net update with the server, thereby achieving communication-efficiency.  Furthermore, only a subset of clients communicate with the server, and this subset may be different at different synchronization times.  The Byzantine clients may collaborate and send arbitrary vectors to the server to disrupt the learning process. To combat the adversary, we employ an efficient high-dimensional robust mean estimation algorithm from Steinhardt et al.~\cite[ITCS 2018]{Resilience_SCV18} at the server to filter-out corrupt vectors; and to analyze the outlier-filtering procedure, we develop a novel matrix concentration result that may be of independent interest.

We provide convergence analyses for strongly-convex and non-convex smooth objectives in the heterogeneous data setting, where different clients may have different local datasets, and we do not make any probabilistic assumptions on data generation. We believe that ours is the first Byzantine-resilient algorithm and analysis with local iterations. We derive our convergence results under minimal assumptions of bounded variance for SGD and bounded gradient dissimilarity (which captures heterogeneity among local datasets). We also extend our results to the case when clients compute full-batch gradients. 
\end{abstract}

{\allowdisplaybreaks

\section{Introduction}\label{sec:intro}

In the {\em federated learning} (FL) paradigm \cite{KonecnyThesis,jakub2016fed,McMahan-FL17,AgnosticFL19}, several clients 
(e.g., mobiles devices, organizations, etc.) collaboratively learn a machine learning model, where the training process is facilitated by the data 
held by the participating clients (without data centralization) and is coordinated by a central server (e.g., the service provider). 
Due to its many advantages over the traditional centralized learning \cite{DeanLargeScale12} 
(e.g., training a machine learning model without collecting the clients' data, which, in addition to reducing the communication load on the network, 
provides a basic level of privacy to clients' data), FL has emerged as an active area of research recently; see \cite{Kairouz-FL-survey19} for a detailed survey.
Stochastic gradient descent (SGD) has become a de facto standard in optimization for training machine learning models at such a large scale \cite{Bottou10,McMahan-FL17,Kairouz-FL-survey19}, where clients iteratively communicate the gradient updates with the central server,
which aggregates the gradients, updates the learning model, and sends the aggregated gradient back to the clients. 
The promise of FL comes with its own set of challenges \cite{Kairouz-FL-survey19}:
{\sf (i)} optimizing with {\em heterogeneous} data at different clients, who may have different local datasets, 
which may be ``non-i.i.d.'', i.e., can be thought of as being generated from different underlying distributions;
{\sf (ii)} slow and unreliable network connections between the server and the clients, so communication in every iteration may not be feasible;
{\sf (iii)} availability of only a subset of clients for training at a given time (maybe due to low connectivity, as clients may be located in different geographic locations); and
{\sf (iv)} robustness against the malicious/Byzantine clients who may send incorrect gradient updates to the central server to disrupt the training process.
In this paper, we propose and analyze a {\em single} SGD algorithm that addresses all these challenges {\em together}. 
First we setup the problem, put our work in context with the related work, and then summarize our contributions.

We consider an empirical risk minimization problem, where data is stored at $R$ clients, 
each having a different dataset (with no probabilistic assumption on data generation); client $r\in[R]$ has dataset $\calD_r$.
Let $F_r:\R^d\to\R$ denote the local loss function associated with the dataset $\calD_r$, which is defined as $F_r(\bx)\triangleq\bbE_{i\in_U[n_r]}[F_{r,i}(\bx)]$,
where $n_r=|\calD_r|$, $i$ is uniformly distributed over $[n_r]\triangleq\{1,2,\hdots,n_r\}$, and $F_{r,i}(\bx)$ is the loss associated with the $i$'th data point at client $r$ with respect to (w.r.t.) $\bx$. Our goal is to solve the following minimization problem:
\begin{align}\label{problem-expression}
\arg \min_{\bx\in \R^d}\Big(F(\bx) \triangleq \frac{1}{R}\sum_{r=1}^R\bbE_{i\in_U[n_r]}[F_{r,i}(\bx)]\Big).
\end{align} 
Let $\bx^*\in\arg\min_{\bx\in\R^d}F(\bx)$ denote a minimizer of the global loss function $F(\bx)$.
In absence of the above-mentioned FL challenges, we can minimize \eqref{problem-expression} using distributed {\em vanilla} SGD, where in any iteration,
server broadcasts the current model parameters to all the clients, each of them then computes a stochastic gradient from its local dataset and
sends it back to the server, who aggregates the received gradients and updates the global model parameters. 
However, this simple solution does not satisfy the FL challenges, as {\em every} client communicates with the server (i.e., no sampling of clients) 
in {\em every} SGD iteration (i.e., no local iterations), and furthermore, this solution
breaks down even with a single malicious client \cite{Krum_Byz17}.

\subsection{Related Work}\label{subsec:related-work}
Recent work has proposed variants of the above-described vanilla SGD that address {\em some} of the FL challenges.
The algorithms in \cite{MadhaviLocalSGD19,MadhaviLocalFL19,scaffold19,TighterTheoryLocalGD20,FedAvgNonIID20,TalwalkarFL20,alibaba_local19,Qsparse-local-sgd19} 
work under different heterogeneity assumptions but do not provide any robustness to malicious clients.
On the other hand, \cite{ChenSX_Byz17,Krum_Byz17,Bartlett-Byz18,Alistarh_Byz-SGD18,SuX_Byz19,Zeno_ByzSGD19,Bartlett-Byz_nonconvex19} provide robustness, but with no local iterations or sampling of clients; furthermore, they assume homogeneous (either same or i.i.d.) data across all clients.
A different line of work \cite{Draco_ByzGD18,Detox_ByzSGD19,DataSoDi_ByzGD19,DataDi_ByzCD19,DataSoDi_arxiv-Byz19,RSA_Byz-SGD19,YinRobustFL19,DataDi_Byz-SGD_Heterogeneous20,Jaggi-Heterogeneous20} provides robustness with heterogeneous data, but without local iterations or sampling of clients, which we briefly explain in the following. \cite{Draco_ByzGD18,Detox_ByzSGD19,DataSoDi_ByzGD19,DataDi_ByzCD19,DataSoDi_arxiv-Byz19} use coding across datasets, which is hard to implement in FL.
\cite{RSA_Byz-SGD19} changes the objective function and adds a regularizer term to combat the adversary. 
\cite{YinRobustFL19} effectively reduces the heterogeneous problem to a homogeneous problem by clustering, and then learning happens within each cluster having homogeneous data. 
\cite{Jaggi-Heterogeneous20} proposed a resampling technique that effectively adapts existing robust algorithms (which might have been designed to work with homogeneous -- identical or i.i.d.\ -- datasets) to work with heterogeneous datasets. 
Note that \cite{Jaggi-Heterogeneous20} provides convergence guarantees of their resampling techniques applied to only {\sc Krum}, which is the robust aggregation rule from \cite{Krum_Byz17}.

\cite{DataDi_Byz-SGD_Heterogeneous20} is the closest related work to ours, in the sense that they also proposed an SGD algorithm on heterogeneous data that uses robust mean estimation subroutines to filter out corrupt gradients and analyzed it under the same minimal assumptions as ours. We want to emphasize that \cite{DataDi_Byz-SGD_Heterogeneous20} does not incorporate local iterations and sampling of clients in their algorithm and analyses, which makes our analyses fundamentally different from theirs. We had to develop new tools (a matrix concentration inequality) to analyze our algorithm, and also the convergence analyses in our paper are very different from those in literature, including that in \cite{DataDi_Byz-SGD_Heterogeneous20}. 
Our analyses differ from those of local SGD without adversaries, as (apart from differing in other technical details) our method requires establishing two separate recurrences, one at synchronization indices and the other one for the rest of the indices.
Our analyses also differ from those of SGD without local iterations and without adversaries, as local SGD causes drift in the local parameter vectors at clients in between any two synchronization indices -- this drift occurs even when all clients have identical data. Note that bounding this drift is necessary for convergence but is non-trivial with heterogeneous data and without having strong assumptions. Our matrix concentration result and its analysis is also very different from that of \cite{DataDi_Byz-SGD_Heterogeneous20}, as we need to prove it in the presence of local iterations.

We believe that ours is the first work that combines {\em local iterations} with {\em Byzantine-resilience} for SGD.\footnote{At the completion of our work, we found that \cite{SLSGD19} also analyzed SGD in the FL setting, but with the following major differences: Not only do they make bounded gradient assumption, the approximation error (even in the Byzantine-free setting) of their solution could be as large as $\calO(D^2+G^2)$, where $G$ is the gradient bound and $D$ is the diameter of the parameter space that contains the optimal parameters $\bx^*$ and all the local parameters $\bx_r^t$ ever emerged at any client $r\in[R]$ in any iteration $t\in[T]$; this, in our opinion, makes the bound vacuous. In optimization, one would ideally like to have the convergence rates depend on diameter of the parameter space with a factor that decays with the number of iterations, e.g., with $\frac{1}{T}$ or $\frac{1}{\sqrt{T}}$, and also see \Theoremref{LocalSGD_convergence}.}
Not only that, we also analyze our algorithm on {\em heterogeneous} data and allow {\em sampling of clients}.
Note that, apart from the notable exception of \cite{DataDi_Byz-SGD_Heterogeneous20}, the earlier work that provides robustness (without local iterations or sampling of clients) either assume homogeneous data across clients \cite{ChenSX_Byz17,Krum_Byz17,Bartlett-Byz18,Alistarh_Byz-SGD18,SuX_Byz19,Bartlett-Byz_nonconvex19} or require strong assumptions, such as the bounded gradient assumption on local functions (i.e., $\|\nabla F_r(\bx)\|\leq G$ for some finite $G$) \cite{Zeno_ByzSGD19}. 
Note that even without robustness, assuming bounded gradients is a common way to make the analysis on heterogeneous data simple \cite{alibaba_local19,FedAvgNonIID20}, as under this assumption, 
we can trivially bound the heterogeneity among local datasets by $\|\nabla F_r(\bx)-\nabla F_s(\bx)\|\leq2G$,%
\footnote{See \cite{TighterTheoryLocalGD20} for a detailed discussion on the inappropriateness of making bounded gradient assumption in heterogeneous data settings and examine the effect of heterogeneity on convergence rates (even without robustness).} which makes handling heterogeneity vacuous.

\subsection{Our Contributions}\label{subsec:our-contributions}
In this paper, we tackle heterogeneity assuming only that the gradient dissimilarity among local datasets is bounded (see \eqref{bounded_local-global}), 
and propose and analyze a Byzantine-resilient SGD algorithm with local iterations and sampling of clients under the bounded variance assumption for SGD (see \eqref{bounded_local-variance}); see \Algorithmref{Byz-Local-SGD}. 
We provide convergence analyses for strongly-convex and non-convex smooth objectives.
Our convergence results are summarized below, where $b$ is the mini-batch size for stochastic gradients, $\sigma^2$ is the variance bound, $\kappa^2$ captures the gradient dissimilarity, $H$ is the number of local iterations in between any two consecutive synchronization indices, $K$ is the number of clients sampled at synchronization times, $\eps<\frac{K}{4R}$ is the fraction of Byzantine clients, and $\eps'$ is any constant such that $(\eps+\eps')\leq\frac{K}{4R}$.

For strongly-convex objectives, our algorithm can find approximate optimal parameters within an error of $\varGamma=\calO\(\frac{H\sigma^2}{b\eps'}\(1+\frac{d}{K}\)(\eps+\eps')+H\kappa^2\)$ exponentially (in $\frac{T}{H}$) fast, and for non-convex objectives, it can reach to a stationary point within the same error $\varGamma$ with a speed of $\frac{1}{\nicefrac{T}{H}}$. 
Note that the convergence rate of {\em vanilla} SGD (i.e., without local iterations and in Byzantine-free settings) decays exponentially (in $T$) fast for strongly-convex objectives and with a speed of $\frac{1}{T}$ for non-convex objectives, whereas,  our convergence rates are affected by the number of local iterations $H$.
This is a result of working with weak assumptions -- if we work with the bounded gradient assumption, 
then we can also get exponential (in $T$) convergence in the strongly-convex case and $\frac{1}{T}$ convergence in the non-convex case.

In the approximation error $\varGamma$, the first error term $\frac{H\sigma^2}{b\eps'}\(1+\frac{d}{K}\)(\eps+\eps')$ mainly arises because of the stochasticity in gradients due to SGD and is equal to zero if we work with full-batch gradients (which gives $\sigma=0$), and the second error term $H\kappa^2$ arises because of heterogeneity in local datasets.
Note that $\varGamma$ only has a linear dependence on $H$.

We also give a simplified analysis of our algorithm with {full-batch} gradients for all three objectives.
See \Theoremref{LocalSGD_convergence} and \Theoremref{full-batch-GD} for our mini-batch SGD and full-batch GD convergence results, respectively.
See a detailed discussion on the approximation error analysis and the convergence rates in \Sectionref{important-remarks}.

To tackle the malicious behavior of Byzantine clients, we borrow tools from recent advances in high-dimensional robust statistics \cite{Robust_mean_LRV16,Resilience_SCV18,DKKLMS19,RecentAdvances_RobustStatistics19};
in particular, we use the polynomial-time outlier-filtering procedure from \cite{Resilience_SCV18}, which was developed for robust mean estimation in high dimensions. In order to use this algorithm, we develop a novel matrix concentration result (see \Theoremref{gradient-estimator}) which may be of independent interest.
For full-batch gradients, we give our matrix concentration result with better guarantees, which can be proved by a much simplified analysis than its mini-batch counterpart; see \Theoremref{gradient-estimator_GD}.

\subsection{Paper Organization}
We describe our algorithm and state the main convergence results in \Sectionref{problem-setup_results}.
We describe the core part of our algorithm, the robust accumulated gradient estimation (RAGE), and our new matrix concentration result in \Sectionref{robust-grad-est} and also prove it there.
We prove our main convergence results for mini-batch SGD in \Sectionref{convex_convergence} and \Sectionref{nonconvex_convergence} and for full-batch SGD in \Sectionref{convergence_full-batch-GD}.

\subsection{Notation}
For any $n\in\mathbb{N}$, we denote the set $\{1,2,\hdots,n\}$ by $[n]$, and for any $n_1,n_2\in\mathbb{N}$ such that $n_1\leq n_2$, we denote the set $\{n_1,n_1+1,\hdots,n_2\}$ by $[n_1:n_2]$.
We denote vectors by bold small letters $\bx,\by$, etc., and matrices by bold capital letters $\bA,\bB$, etc.
For any finite set $\calK$, we write $k\in_U\calK$ to denote that $k$ is chosen uniformly at random from $\calK$.
All vector norms in this paper are $\ell_2$ norms, and for convenience, we simply denote them by $\|\cdot\|$.
For a square matrix $\bA$, we write $\lambda_{\max}(\bA)$ to denote the largest eigenvalue of $\bA$.

\section{Problem Setup and Our Results}\label{sec:problem-setup_results}

In this section, we state our assumptions, describe the adversary model and our algorithm, and state our main convergence results.

\subsection{Assumptions}\label{subsec:assumptions}
As mentioned in \Sectionref{intro}, we make minimal assumptions to analyze our algorithm. Our first assumption is a standard one in SGD, which assumes bounded variance for stochastic gradients. Our second assumption is for heterogeneous data and assumes that the heterogeneity in different local datasets is bounded.
\begin{assumption}[Bounded local variances]\label{assump:bounded_local-variance}
The stochastic gradients sampled from any local dataset have uniformly bounded variance over $\R^d$, i.e., there exists a finite $\sigma\geq0$, such that
\begin{align}\label{bounded_local-variance}
\bbE_{i\in_U[n_r]}\|\nabla F_{r,i}(\bx) - \nabla F_r(\bx)\|^2 \leq \sigma^2, \quad \forall \bx\in\R^d, r\in[R].  
\end{align}
\end{assumption}
It will be helpful to formally define mini-batch stochastic gradients, where instead of computing stochastic gradients based on just one data point, 
each client selects a subset of size $b$ uniformly at random from its own local dataset and computes the average of $b$ gradients.
For any $\bx\in\R^d, r\in[R], b\in[n_r]$, consider the following set 
\begin{align}\label{bigger_set}
\calF_r^{\otimes b}(\bx):=\left\{ \frac{1}{b}\sum_{i\in\calH_b}\nabla F_{r,i}(\bx) :  \calH_b\in \binom{[n_r]}{b} \right\}.
\end{align}
Note that $\bg_r(\bx)\in_U\calF_r^{\otimes b}(\bx)$ is a mini-batch stochastic gradient with batch size $b$ at client $r$. It is not hard to see the following:
\begin{align}
\bbE\left[\bg_r(\bx)\right] &= \nabla F_r(\bx), \quad \forall \bx\in\R^d, r\in[R] \label{same_mean} \\
\bbE\left\|\bg_r(\bx) - \nabla F_r(\bx)\right\|^2 &\leq \frac{\sigma^2}{b}, \quad \forall \bx\in\R^d, r\in[R]\label{reduced_variance}
\end{align}
where \eqref{same_mean} says that $\bg_r(\bx)$ is an unbiased gradient and \eqref{reduced_variance} says that the variance of mini-batch stochastic gradients reduces by the same factor as the batch size.
Though the bound in \eqref{reduced_variance} goes down with $b$, it does not become zero when we compute full-batch gradients, which uses all $n_r$ data points. This is because \eqref{reduced_variance} only uses that the clients sample $b$ data points {\em with} replacement. However, in reality, since this sampling is done {\em without} replacement, we can show a finer variance bound of $\bbE\left\|\bg_r(\bx) - \nabla F_r(\bx)\right\|^2 \leq \frac{(n_r-b)}{b(n_r-1)}\sigma^2$; see \cite{mini-batch_variance17} for a proof. 
We can slightly improve our results by using this finer variance bound instead of \eqref{reduced_variance} everywhere in this paper, but, for simplicity, we only use the weaker bound \eqref{reduced_variance} throughout.

\begin{assumption}[Bounded gradient dissimilarity]\label{assump:gradient-dissimilarity}
The difference between the local gradients $\nabla F_r(\bx), r\in[R]$ and the global gradient $\nabla F(\bx)=\frac{1}{R}\sum_{r=1}^R\nabla F_r(\bx)$ is uniformly bounded over $\R^d$ for all clients, i.e., there exists a finite $\kappa$, such that
\begin{align}\label{bounded_local-global}
\|\nabla F_r(\bx) - \nabla F(\bx)\|^2 \leq \kappa^2, \quad \forall \bx\in\R^d, r\in[R].  
\end{align}
\end{assumption}
In \Assumptionref{gradient-dissimilarity}, $\kappa$ quantifies the bounded deviation between the local loss functions $F_r, r\in[R]$ and the global loss function $F$;  
see also \cite{Momentum_linear-speedup19,LocalDecentralized_19}, where this assumption has been used in heterogeneous data settings in decentralized SGD without Byzantine clients.
The gradient dissimilarity bound in \eqref{bounded_local-global} can be seen as a {\em deterministic} condition on local datasets, under which we derive our results. 

\subsubsection{Need of \Assumptionref{gradient-dissimilarity}}
For any method that filters out malicious updates from the clients and work with heterogeneous (``non i.i.d.'') data, as the server does not know the identities of the adversarial clients, we need to have some regularity condition relating the datasets, and we believe \Assumptionref{gradient-dissimilarity} is a natural way to model that. \Assumptionref{gradient-dissimilarity} intuitively captures the heterogeneity among local datasets, without making any statistical assumptions on the data. To see the necessity of bounding heterogeneity even without adversary, note that we allow clients to perform local SGD steps, where, in between any two synchronization indices, clients compute gradients from their local datasets and update their local parameter vectors; as a result, their local parameter vectors can drift away from each other. This drift needs to be bounded for convergence analyses, and if we do not assume bounded heterogeneity, it is impossible to bound this drift. As we have discussed at the end of \Subsectionref{related-work}, \Assumptionref{gradient-dissimilarity} is much weaker than the bounded gradient assumption, which not only makes bounding the drift (and the convergence analyses) trivial, but also obscure the dependence of the convergence bounds on the heterogeneity of datasets, which is clearly brought out in our convergence results.

\subsubsection{Bounds on $\sigma^2$ and $\kappa^2$ in the statistical heterogeneous model}
Since all results (matrix concentration and convergence) in this paper are given in terms of $\sigma^2$ and $\kappa^2$, to show the clear dependence of our results on the dimensionality of the problem, we can bound these quantities in the statistical {\em heterogeneous} data model under different distributional assumptions on local gradients.
For the variance bound \eqref{bounded_local-variance}, it was shown in \cite[Theorem 7]{DataDi_Byz-SGD_Heterogeneous20} that if local gradients have sub-Gaussian distribution, then $\sigma=\calO\(\sqrt{d\log(d)}\)$. For the gradient dissimilarity bound \eqref{bounded_local-global}, it was shown in \cite[Theorem 6]{DataDi_Byz-SGD_Heterogeneous20} that if either the local gradients have sub-exponential distribution and each client has at least $n=\Omega\(d\log(nd)\)$ data points or local gradients have sub-Gaussian distribution and $n\in\mathbb{N}$ is arbitrary, then $\kappa \leq \kappa_{\text{mean}} + \calO\(\sqrt{\frac{d\log(nd)}{n}}\)$, where $\kappa_{\text{mean}}$ denotes the distance of the expected local gradients from the global gradient.

\subsection{Adversary Model}
We assume that an $\eps$ fraction of $R$ clients are malicious; as we see later, we can tolerate $\eps<\frac{K}{4R}$,\footnote{Actually, we can tolerate $\eps<\frac{1}{4}$ fraction of malicious clients from the $K$ clients that we select;  so, $\eps<\frac{K}{4R}$ is a worst case bound in case we sample {\em all} the malicious clients in a selection, which is an unlikely event.} 
where $K\leq R$ is the number of clients sampled at synchronization indices. The malicious clients can {\em collaborate} and {\em arbitrarily} deviate from their pre-specified programs: In any SGD iteration, instead of sending true stochastic gradients, corrupt clients may send adversarially chosen vectors (they may not even send anything if they wish, in which case, the server can treat them as {\em erasures} and replace them with a fixed value). Note that, in the erasure case, server knows which clients are corrupt; whereas, in the Byzantine problem, server does not have this information.

\subsection{Main Results}\label{subsec:main-results}
Let $\I_T=\{t_1,t_2,\hdots,t_k,\hdots\}$, with $t_1=0$, denote the set of synchronization indices at which clients communicate their net updates with the server. Let
$H$ denote the difference between any two consecutive indices, i.e., every worker performs the same number $H$ of local iterations between any two consecutive synchronization indices.
At synchronization indices, server samples a subset of $K$ clients (denoted by $\calK\subseteq[R]$) and sends the global model (denoted by $\bx$) to them; each client $r\in\calK$ updates its local model $\bx_r$ by taking SGD steps based on its local dataset until the next synchronization time, when all clients in $\calK$ send their local models to the server. Note that some of these clients may be corrupt and may send arbitrary vectors.
Server employs a decoding algorithm RAGE\footnote{\label{foot:RAGE-vs-RME}Our decoding algorithm, which we call RAGE, is the same as the robust mean estimation algorithm proposed by Steinhardt et al.~\cite{Resilience_SCV18}. We gave it a different name, as we use it in a much more general FL setting of running SGD with local iterations on heterogeneous data. Note that the same algorithm has also been used in \cite{SuX_Byz19,Bartlett-Byz_nonconvex19} in the context of Byzantine-robust {\em full batch} gradient descent {\em without} local iterations, assuming  {\em homogeneous} i.i.d.\ data, whereas, we employ that algorithm in the FL setting, which makes its analysis significantly more challenging.} and update the global model $\bx$ based on that.
\begin{remark}
Note that the only disruption that the corrupt clients can cause in the training process is during the gradient aggregation at synchronization indices by sending adversarially chosen vectors to the server, and we give unlimited power to the adversary for that. Because of this and for the purpose of analysis, we can assume, without loss of generality, that in between the synchronization indices, the corrupt clients sample stochastic gradients and update their local parameters honestly.
\end{remark}
We present our Byzantine-resilient SGD algorithm with local iterations in \Algorithmref{Byz-Local-SGD}.

\begin{algorithm}[t]
   \caption{Byzantine-Resilient SGD with Local Iterations}\label{algo:Byz-Local-SGD}
\begin{algorithmic}[1]
   \STATE {\bf Initialize.} Set $t:=0$, $\bx_r^0 := \bzero, \forall r\in[R]$, and $\bx:=\bzero$.  
   Here, $\bx$ denotes the global model and $\bx_r^0$ denotes the local model at client $r$ at time 0.
   Fix a constant step-size $\eta$ and a mini-batch size $b$. 
   \WHILE{($t\leq T$)}
   \STATE Server selects an arbitrary subset of clients $\calK\subseteq[R]$ of size $|\calK|=K$ and sends $\bx$ to all clients in $\calK$.
  \STATE \textbf{All clients $r\in\calK$ do in parallel:}
  \STATE Set $\bx_r^t=\bx$.
  \WHILE{(true)}
   \STATE Take a mini-batch stochastic gradient $\bg_r(\bx_r^t) \in_U \calF_r^{\otimes b}(\bx_r^t)$ and update the local model:
   \begin{align*}
   \bx_r^{t+1} \leftarrow \bx_r^{t} - \eta\bg_r(\bx_r^{t}); \quad t \leftarrow (t+1).
   \end{align*}
   \vspace{-0.3cm}
   \IF{($t\in\I_T$)}
      \STATE Let $\btlx_{r}^t=\bx_{r}^t$, if client $r$ is honest, otherwise $\btlx_{r}^t$ can be an arbitrary vector in $\R^d$.
      \STATE Send $\btlx_{r}^t$ to the server and break the inner {\bf while} loop.

   \ENDIF
   \ENDWHILE
   
   \STATE \textbf{At Server:}
   \STATE Receive $\{\btlx_{r},r\in\calK\}$ from the clients in $\calK$.
   \STATE For every $r\in\calK$, let $\btlg_{r,\accu} := (\btlx_r-\bx)/\eta$. 
   \STATE Apply the decoding algorithm {\sc RAGE} (see \Algorithmref{robust-grad-est} in \Subsectionref{robust-grad-est_algo}) on $\{\btlg_{r,\accu}, r\in\calK\}$. 
   Let $$\btg_\accu := \textsc{RAGE}(\btlg_{r,\accu}, r\in\calK).$$ \vspace{-0.3cm}
   \STATE Update the global model $\bx \leftarrow \bx - \eta\btg_{\accu}$. 
   \ENDWHILE
\end{algorithmic}
\end{algorithm}

Before we present our results, we need some definitions.
\begin{itemize}
\item {\bf $L$-smoothness:} A function $F:\R^d\to\R$ is called $L$-smooth over $\R^d$, if for every $\bx,\by\in\R^d$, we have $\|\nabla F(\bx) - \nabla F(\by)\| \leq L\|\bx-\by\|$ (this property is also known as $L$-Lipschitz gradients). This is also equivalent to $F(\by) \leq F(\bx) + \langle\nabla F(\bx), \by-\bx\rangle + \frac{L}{2}\|\bx-\by\|^2$.
\item {\bf $\mu$-strong convexity:} A function $F:\R^d\to\R$ is called $\mu$-strongly convex over $\R^d$ for $\mu\geq0$, if for every $\bx,\by\in\R^d$, we have $F(\by) \geq F(\bx) + \langle\nabla F(\bx), \by-\bx\rangle + \frac{\mu}{2}\|\bx-\by\|^2$.
\end{itemize}
All convergence results in this paper only require properties of the global loss function $F$; the local loss functions $F_r,r\in[R]$ may be arbitrary. For example, in the smooth strongly-convex case, we only require $F$ to be smooth and strongly-convex, and $F_r,r\in[R]$ may be arbitrary. Similarly for the non-convex case.

Our convergence results are for strongly-convex and non-convex smooth objectives.
\begin{theorem}[Mini-Batch Local Stochastic Gradient Descent]\label{thm:LocalSGD_convergence}
Suppose an $\epsilon>0$ fraction of clients are adversarially corrupt.
Let $\calK_t$ denote the set of $K$ clients that are active at any given time $t\in[0:T]$.
For a global objective function $F:\R^d\to\R$, let \Algorithmref{Byz-Local-SGD} generate a sequence of iterates $\{\bx_r^t : t\in[0:T], r\in\calK_t\}$ when run with a fixed step-size $\eta=\frac{1}{8HL}$. 
Fix an arbitrary constant $\eps'>0$. If $\eps\leq\frac{K}{4R} - \eps'$, then with probability at least $1-\frac{T}{H}\exp(-\frac{\eps'^2(1-\eps)K}{16})$, the sequence of average iterates $\{\bx^t=\frac{1}{K}\sum_{r\in\calK_t}\bx_r^t : t\in[0:T]\}$
satisfy the following convergence guarantees: 
\begin{itemize}
\item {\bf Strongly-convex:} If $F$ is $L$-smooth for $L\geq0$ and $\mu$-strongly convex for $\mu>0$, we get:
\begin{align}
\bbE\left\|\bx^{T} - \bx^*\right\|^2 \leq \(1-\frac{\mu}{16HL}\)^T\left\|\bx^{0} - \bx^*\right\|^2 + \frac{13}{\mu^2}\varGamma.\label{convex_convergence_rate}
\end{align}
\item {\bf Non-convex:} If $F$ is $L$-smooth for $L\geq0$, we get:
\begin{align}
\frac{1}{T}\sum_{t=0}^T\bbE\left\|\nabla F(\bx^{t})\right\|^2 \leq \frac{16HL}{T}\left[\bbE[F(\bx^0)] - \bbE[F(\bx^{*})]\right] + \frac{9}{2}\varGamma. \label{nonconvex_convergence_rate}
\end{align}
\end{itemize}
In \eqref{convex_convergence_rate}, \eqref{nonconvex_convergence_rate}, $\varGamma=\(\frac{3\varUpsilon^2}{H} + \frac{11H\sigma^2}{b} + 36H\kappa^2\)$ with $\varUpsilon^2 = \calO\left(\sigma_0^2(\eps+\eps')\right)$, where $\sigma_0^2 = \frac{25H^2\sigma^2}{b\eps'}\left(1 + \frac{4d}{3K}\right) + 28H^2\kappa^2$,
and expectation is taken over the sampling of mini-batch stochastic gradients. 

\end{theorem}
We prove \eqref{convex_convergence_rate} and \eqref{nonconvex_convergence_rate} in \Sectionref{convex_convergence} and \Sectionref{nonconvex_convergence}, respectively.
In addition to other complications arising due to handing Byzantine clients together with local iterations, our proof deviates from the standard proofs for local SGD without adversary, as we need to show two recurrences, one at synchronization indices and the other at non-synchronization indices.
This is because at synchronization indices, 
server performs decoding to filter-out the corrupt clients, while at other indices there is no decoding, as there is no communication. 

The failure probability of our algorithm is at most $\frac{T}{H}\exp(-\frac{\eps'^2(1-\eps)K}{16})$, which though scales linearly with $T$, 
also goes down exponentially with $K$. 
As a result, in settings such as federated learning, where number of clients could be very large (e.g., in millions) and server samples a few thousand clients, we can get a very small probability of error, 
even if we run our algorithm for a very long time. 
Note that the error probability is due to the {\em stochastic} sampling of gradients,
and if we want a ``zero'' probability of error, we can run full-batch gradient descent, for which we get the following result, which we prove in \Sectionref{convergence_full-batch-GD} with a much simplified analysis than that of \Theoremref{LocalSGD_convergence}.

\begin{theorem}[Full-Batch Local Gradient Descent]\label{thm:full-batch-GD}
In the same setting as that of \Theoremref{LocalSGD_convergence}, except for that we run \Algorithmref{Byz-Local-SGD} with a fixed step-size $\eta=\frac{1}{5HL}$, and in any iteration, instead of sampling mini-batch stochastic gradients, every honest client takes full-batch gradients from their local datasets.
If $\eps\leq\frac{K}{4R}$, then with probability 1, the sequence of average iterates $\{\bx^t=\frac{1}{K}\sum_{r\in\calK_t}\bx_r^t : t\in[0:T]\}$
satisfy the following convergence guarantees:
\begin{itemize}
\item {\bf Strongly-convex:} If $F$ is $L$-smooth for $L\geq0$ and $\mu$-strongly convex for $\mu>0$, we get:
\begin{align}
\|\bx^{T} - \bx^* \|^2 \leq \left(1-\frac{\mu}{10HL}\right)^T\|\bx^0-\bx^*\|^2 + \frac{14}{\mu^2}\varGamma_{\GD}. \label{convex-GD_convergence_rate}
\end{align}
\item {\bf Non-convex:} If $F$ is $L$-smooth for $L\geq0$, we get:
\begin{align}
\frac{1}{T}\sum_{t=0}^T\left\|\nabla F(\bx^{t})\right\|^2 \leq \frac{10HL}{T}\left[F(\bx^0) - F(\bx^{*}) \right] + \frac{24}{5}\varGamma_{\GD}. \label{nonconvex-GD_convergence_rate}
\end{align}
\end{itemize}
In \eqref{convex-GD_convergence_rate}, \eqref{nonconvex-GD_convergence_rate}, 
$\varGamma_{\GD} = \frac{2\varUpsilon_{\GD}^2}{H}+25H\kappa^2$, where $\varUpsilon_{\GD} = \calO\left(H\kappa\sqrt{\eps}\right)$.
\end{theorem}

\subsection{Important Remarks About \Theoremref{LocalSGD_convergence} and \Theoremref{full-batch-GD}}
\label{sec:important-remarks}

\paragraph{Analysis of the approximation error.}
In \Theoremref{LocalSGD_convergence}, the approximation error $\varGamma$ essentially consists of two types of error terms:
$\varGamma_1=\calO\(\frac{H\sigma^2}{b\eps'}\left(1 + \frac{4d}{3K}\right)(\eps+\eps')\)$ and $\varGamma_2=\calO(H\kappa^2)$,
where $\varGamma_1$ arises due to stochastic sampling of gradients and $\varGamma_2$ arises due to dissimilarity in the local datasets.
Observe that $\varGamma_1$ decreases as we increase the batch size $b$ of stochastic gradients and becomes zero if we take full-batch gradients (which implies $\sigma=0$), as is the case in \Theoremref{full-batch-GD}.
Note that both $\varGamma_1$ and $\varGamma_2$ have a linear dependence on the number of local iterations $H$. 
Observe that since we are working with heterogeneous datasets, the presence of gradient dissimilarity bound $\kappa^2$ (which captures the heterogeneity) in the approximation error is inevitable, and will always show up when bounding the deviation of the true ``global'' gradient from the decoded one in the presence of Byzantine clients, even when $H=1$.
 
\paragraph{Convergence rates.} 
In the strongly-convex case, \Algorithmref{Byz-Local-SGD} approximately finds the optimal parameters $\bx^*$ (within $\varGamma$ error) with $\left(1-\frac{\mu}{cHL}\right)^T$ speed, where $c=16$ for SGD and $c=10$ for GD. 
Note that $\left(1-\frac{\mu}{cHL}\right)^T \leq \exp^{-\frac{\mu}{cL}\frac{T}{H}}$, where the inequality follows from $(1-\frac{1}{x})^x \leq \frac{1}{e}$. This implies that the convergence rate in this case is exponentially fast (but in $\frac{T}{H}$).
In the non-convex case, \Algorithmref{Byz-Local-SGD} reaches to a stationary point (within $\varGamma$ error) with a speed of $\frac{1}{\nicefrac{T}{H}}$. 
Note that the convergence rate of {\em vanilla} SGD (i.e., without local iterations and in Byzantine-free settings) is exponentially fast (in $T$) for strongly-convex objectives and with a speed of $\frac{1}{T}$ for non-convex objectives, whereas,  our convergence rates are affected by the number of local iterations $H$.
The reason for this is precisely because, under standard SGD assumptions we need $\eta\leq\frac{1}{8HL}$ to bound the drift in local parameters across different clients; see \Lemmaref{bounded_local-global-params-app}. 
Instead, if we had assumed a stronger bounded gradient assumption (which trivially bound the heterogeneity, as explained at the end of \Subsectionref{related-work}), then \Lemmaref{bounded_local-global-params-app} 
would hold for a constant step-size that does not depend on $H$ (e.g., $\eta=\frac{1}{2L}$ would suffice), 
which would lead to an exponentially fast (in $T$) convergence for strongly-convex objectives and $\frac{1}{T}$ convergence rate for non-convex objectives.

\section{Robust Accumulated Gradient Estimation (RAGE)}\label{sec:robust-grad-est}
In this section, we provide our main result on robust accumulated gradient estimation (RAGE), which is the subroutine for robustly estimating the average of uncorrupted {\em accumulated} gradients at every synchronization index; see \Footnoteref{RAGE-vs-RME}.
First we setup the notation. 
Let \Algorithmref{Byz-Local-SGD} generate a sequence of iterates $\{\bx_r^t:t\in[0:T],r\in\calK_t\}$ when run with a fixed step-size $\eta$ satisfying 
$\eta\leq\frac{1}{8HL}$, where $\calK_t$ denotes the set of $K$ clients that are active at time $t\in[0:T]$.
Take any two consecutive synchronization indices $t_k,t_{k+1}\in\I_T$. Note that $|t_{k+1}-t_k|\leq H$.
For an honest client $r\in\calK_{t_k}$, let $\bg_{r,\accu}^{t_k,t_{k+1}}:=\sum_{t=t_k}^{t_{k+1}-1}\bg_r(\bx_r^t)$ denote the sum of local mini-batch stochastic gradients sampled by client $r$ between time $t_k$ and $t_{k+1}$, where $\bg_r(\bx_r^t)\in_U\calF_r^{\otimes b}(\bx_r^t)$ satisfies \eqref{same_mean}, \eqref{reduced_variance}.
At iteration $t_{k+1}$, every honest client $r\in\calK_{t_k}$ reports its local model $\bx_r^{t_{k+1}}$ to the server, 
from which server computes $\bg_{r,\accu}^{t_k,t_{k+1}}$ (see line 15 of \Algorithmref{Byz-Local-SGD}), 
whereas, the corrupt clients may report arbitrary and adversarially chosen vectors in $\R^d$.
Server does not know the identity of the corrupt clients, and its goal is to produce an estimate $\btg_{\accu}^{t_k,t_{k+1}}$ of the average accumulated gradients from honest clients as best as possible.

To this end, first we show that there exists a large subset $\calS\subseteq\calK_{t_k}$ of accumulated gradients from honest clients that are concentrated around their average, i.e., have bounded empirical covariance. Once we have shown that, then we will use the polynomial-time outlier-filtering algorithm from \cite{Resilience_SCV18} to estimate the average of the accumulated gradients in $\calS$.
Our main result on RAGE is as follows:
\begin{theorem}[Robust Accumulated Gradient Estimation]\label{thm:gradient-estimator}
Suppose an $\eps$ fraction of $K$ clients that communicate with the server are corrupt. 
In the setting described above, suppose we are given $K\leq R$ accumulated gradients $\btlg_{r,\emph{\accu}}^{t_k,t_{k+1}}, r\in\calK_{t_k}$ in $\R^d$, where $\btlg_{r,\emph{\accu}}^{t_k,t_{k+1}}=\bg_{r,\emph{\accu}}^{t_k,t_{k+1}}$ if the $r$'th client is honest, otherwise can be arbitrary. 
For any constant $\eps'>0$, if $(\eps+\eps')\leq\frac{1}{4}$, then we have: 
\begin{enumerate}
\item {\bf Matrix concentration:} With probability $1-\exp(-\frac{\eps'^2(1-\eps)K}{16})$, there exists a subset $\calS\subseteq\calK_{t_k}$ of {\em uncorrupted} gradients of size $(1-(\eps+\eps'))K \geq \frac{3K}{4}$, such that
\begin{align}\label{mat-concen_gradient-estimator}
\lambda_{\max}\(\frac{1}{|\calS|}\sum_{i\in\calS} \(\bg_i - \bg_{\calS}\)\(\bg_i - \bg_{\calS}\)^T\) \leq \frac{25H^2\sigma^2}{b\eps'}\left(1 + \frac{4d}{3K}\right) + 28H^2\kappa^2,
\end{align}
where, for $i\in\calS$, $\bg_i=\bg_{i,\emph{\accu}}^{t_k,t_{k+1}}, \bg_{\calS}=\frac{1}{|\calS|}\sum_{i\in\calS}\bg_{i,\emph{\accu}}^{t_k,t_{k+1}}$;
and $\lambda_{\max}$ denotes the largest eigenvalue. 

\item {\bf Outlier-filtering algorithm:} We can find an estimate $\btg$ of $\bg_{\calS}$ in polynomial-time with probability 1, such that $\left\| \btg - \bg_{\calS} \right\| \leq \calO\left(\sigma_0\sqrt{\eps+\eps'}\right)$, where $\sigma_0^2 = \frac{25H^2\sigma^2}{b\eps'}\left(1 + \frac{4d}{3K}\right) + 28H^2\kappa^2$.
\end{enumerate}
\end{theorem}

Proving the matrix concentration bound stated in the first part of \Theoremref{gradient-estimator} is non-trivial and we prove it separately in \Subsectionref{matrix-concentration}.
For the second part, we use the polynomial-time outlier-filtering procedure of \cite{Resilience_SCV18}, 
which is a robust mean estimation algorithm, that takes a collection of vectors as input, out of which an unknown large subset (at least a $\frac{3}{4}$-fraction) is promised to be well-concentrated around its sample mean (i.e., has a bounded covariance), and outputs an estimate of the sample mean of the vectors in that subset.
For completeness, we describe this procedure in \Subsectionref{robust-grad-est_algo} and refer the reader to \cite[Appendices E, F]{DataDi_Byz-SGD_Heterogeneous20} for more details.

Note that the same filtering procedure has also been used in \cite{SuX_Byz19,Bartlett-Byz_nonconvex19} in the context of Byzantine-robust {\em full batch} gradient descent {\em without} local iterations for minimizing the population risk, assuming  {\em homogeneous} i.i.d.\ data.
Our setting is very different from theirs, as we minimize the empirical risk by mini-batch {\em stochastic} gradient descent {\em with} local iterations on {\em heterogeneous} data. 
They also derived a matrix-concentration result, whose need arises because they minimize the population risk, whereas, we need a matrix concentration bound because we use SGD. On top of that our setting is much more complicated than theirs, as clients have heterogeneous data and do not communicate with the server in every iteration. As a result, as opposed to their matrix concentration bound (which they proved assuming sub-exponential/sub-Gaussian distribution on local gradients and also assuming i.i.d.\ data across clients), our matrix concentration result is of a very different nature, and we use entirely different tools to derive that.

\subsection{Matrix Concentration}\label{subsec:matrix-concentration}
Now we prove the first part of \Theoremref{gradient-estimator}. For that, we need to show an existence of a subset $\calS$ of the $K$ accumulated gradients (out of which an $\eps<\frac{1}{4}$ fraction is corrupted) that has good concentration, as quantified by the matrix concentration bound in \eqref{mat-concen_gradient-estimator}. 
To prove this, we use a separate matrix concentration result stated in the following lemma from \cite{DataDi_Byz-SGD_Heterogeneous20}.

\begin{lemma}[Lemma 1 in \cite{DataDi_Byz-SGD_Heterogeneous20}]\label{lem:subset_variance}
Suppose there are $m$ independent distributions $p_1,p_2,\hdots,p_m$ in $\R^d$ such that $\bbE_{\by\sim p_i}[\by]=\vct{\mu}_i, i\in[m]$ and each $p_i$ has a bounded variance in all directions, i.e., $\bbE_{\by\sim p_i}[\langle \by - \vct{\mu}_i, \bv\rangle^2] \leq \sigma_{p_i}^2, \forall \bv\in\R^d, \|\bv\|=1$. Take any $\eps'>0$. Then, given $m$ independent samples $\by_1,\by_2,\hdots,\by_m$, where $\by_i\sim p_i$, with probability $1- \exp(-\eps'^2m/16)$, there is a subset $\calS$ of $(1-\eps')m$ points such that
\begin{align*}
\lambda_{\max}\(\frac{1}{|\calS|}\sum_{i\in\calS} \(\by_i-\vct{\mu}_i\)\(\by_i-\vct{\mu}_i\)^T \) \leq \frac{4\sigma_{p_{\max}}^2}{\eps'}\left(1 + \frac{d}{(1-\eps')m}\right), \ \  \text{ where } \sigma_{p_{\max}}^2=\max_{i\in[m]}\sigma_{p_{i}}^2.
\end{align*}
\end{lemma}
Now we prove the first part of \Theoremref{gradient-estimator} with the help of \Lemmaref{subset_variance}. 

Let $t_k,t_{k+1}\in\I_T$ be any two consecutive synchronization indices. 
For $i\in\calK_{t_k}$ corresponding to an honest client, let $Y_{i}^{t_k},Y_{i}^{t_k+1},\hdots,Y_{i}^{t_{k+1}-1}$ be a sequence of $(t_{k+1}-t_k)\leq H$ (dependent) random variables,
where, for any $t\in[t_k:t_{k+1}-1]$, the random variable $Y_{i}^t$ is distributed as 
\begin{align}\label{distr-Yij_proof-matrix-concen-app}
Y_{i}^t \sim \text{Unif}\Big(\calF_i^{\otimes b}\big(\bx_i^{t}\big(\bx_i^{t_k},Y_{i}^{t_k},\hdots,Y_{i}^{t-1}\big)\big)\Big). 
\end{align}
Here, $Y_{i}^t$ is a random variable that corresponds to the stochastic sampling of mini-batch gradients from the set $\calF_i^{\otimes b}\big(\bx_i^{t}\big(\bx_i^{t_k},Y_{i}^{t_k},\hdots,Y_{i}^{t-1}\big)\big)$, which itself depends on the local parameters $\bx_i^{t_k}$ (which is a deterministic quantity) at the last synchronization index and the past realizations of $Y_{i}^{t_k},\hdots,Y_{i}^{t-1}$. This is because the evolution of local parameters $\bx_i^{t}$ depends on $\bx_i^{t_k}$ and the choice of gradients in between time indices $t_k$ and $t-1$. Now define $Y_i:=\sum_{t=t_k}^{t_{k+1}-1} Y_{i}^t$; and let $p_i$ be the distribution of $Y_i$. This is the distribution $p_i$ we will take when using \Lemmaref{subset_variance}.

\begin{claim}\label{claim:reduced_variance_H-iters-app}
For any honest client $i\in\calK_{t_k}$, we have $\bbE\|Y_i-\bbE[Y_i]\|^2\leq\frac{H^2\sigma^2}{b}$, where expectation is taken over sampling stochastic gradients by client $i$ between synchronization indices $t_k$ and $t_{k+1}$.
\end{claim}

\Claimref{reduced_variance_H-iters-app} is proved in \Appendixref{remaining_part1-robust-grad}.

It is easy to see that the hypothesis of \Lemmaref{subset_variance} is satisfied with $\bmu_i=\bbE[Y_i], \sigma_{p_i}^2 = \frac{H^2\sigma^2}{b}$ for all honest clients $i\in\calK_{t_k}$ (note that $p_i$ is the distribution of $Y_i$):
\begin{align*}
\bbE_{\by_i\sim p_i}[\langle \by_i - \bbE[\by_i], \bv\rangle^2] \ \stackrel{\text{(d)}}{\leq} \ 
 \bbE[\| \by_i - \bbE_{\by_i\sim p_i}[\by_i]\|^2]\cdot \|\bv\|^2 \ \stackrel{\text{(e)}}{\leq} \ \frac{H^2\sigma^2}{b}, 
\end{align*}
where (d) follows from the Cauchy-Schwarz inequality and (e) follows from \Claimref{reduced_variance_H-iters-app} and $ \l\| \bv\| \leq 1$.

We are given $K$ different (summations of $H$) gradients, out of which at least $(1-\eps)K$ are according to the correct distribution. 
By considering only the uncorrupted gradients (i.e., taking $m=(1-\eps)K$), we have from 
\Lemmaref{subset_variance} that there exists a subset $\calS\subseteq\calK_{t_k}$ of $K$ gradients of size $(1-\eps')(1-\eps)K \geq (1-(\eps+\eps'))K \geq \frac{3K}{4}$ (where in the last inequality we used $(\eps+\eps')\leq\frac{1}{4}$) that satisfies 
\begin{align}\label{bounded_subset_variance-app}
\lambda_{\max}\(\frac{1}{|\calS|}\sum_{i\in\calS} \(\by_i - \bbE[\by_i]\)\(\by_i - \bbE[\by_i]\)^T \) \leq \frac{4H^2\sigma^2}{b\eps'}\left(1 + \frac{4d}{3K}\right).
\end{align}
Note that \eqref{bounded_subset_variance-app} bounds the deviation of the points in $\calS$ from their respective means $\bbE[\by_i]$. However, in \eqref{mat-concen_gradient-estimator}, we need to bound the deviation of the points in $\calS$ from their sample mean $\frac{1}{|\calS|}\sum_{i\in\calS}\by_i$.
As it turns out, due to our use of local iterations, bounding this requires a substantial amount of technical work, 
which we do in the rest of this subsection.

From the alternate definition of the largest eigenvalue of symmetric matrices $\bA\in\R^{d\times d}$, we have
\begin{align}\label{alternate-defn_max-eigenvalue}
\lambda_{\max}(\bA)=\sup_{\bv\in\R^d,\|\bv\|=1}\bv^T\bA\bv.
\end{align}
Applying this with $\bA=\frac{1}{|\calS|}\sum_{i\in\calS} \(\by_i-\bbE[\by_i]\)\(\by_i-\bbE[\by_i]\)^T$, 
we can equivalently write \eqref{bounded_subset_variance-app} as
\begin{align}\label{bounded_subset_variance_2-app}
\sup_{\bv\in\R^d:\|\bv\|=1}\(\frac{1}{|\calS|}\sum_{i\in\calS} \langle\by_i-\bbE[\by_i], \bv \rangle^2\) \leq \widehat{\sigma}_0^2 := \frac{4H^2\sigma^2}{b\eps'}\left(1 + \frac{4d}{3K}\right).
\end{align}
Define $\by_{\calS}:=\frac{1}{|\calS|}\sum_{i\in\calS} \by_i$ to be the sample mean of the points in $\calS$. Take an arbitrary $\bv\in\R^d$ such that $\|\bv\|=1$.
\begin{align}
\frac{1}{|\calS|}\sum_{i\in\calS} &\langle \by_i-\by_{\calS}, \bv \rangle^2 = \frac{1}{|\calS|}\sum_{i\in\calS} \left[\langle \by_i-\bbE[\by_i], \bv \rangle + \langle \bbE[\by_i] - \by_{\calS}, \bv \rangle\right]^2 \notag \\
&\leq \frac{2}{|\calS|}\sum_{i\in\calS} \langle \by_i-\bbE[\by_i], \bv \rangle^2 + \frac{2}{|\calS|}\sum_{i\in\calS} \langle \bbE[\by_i] - \by_{\calS}, \bv \rangle^2 \tag{using $(a+b)^2 \leq 2a^2 + 2b^2$} \\
\intertext{Using \eqref{bounded_subset_variance_2-app} to bound the first term, we get}
&\leq 2\widehat{\sigma}_0^2 +  \frac{2}{|\calS|}\sum_{i\in\calS} \Big\langle \bbE[\by_i] - \frac{1}{|\calS|}\sum_{j\in\calS}\by_j, \bv \Big\rangle^2 
= 2\widehat{\sigma}_0^2 +  \frac{2}{|\calS|}\sum_{i\in\calS} \Big[\frac{1}{|\calS|}\sum_{j\in\calS}\langle \by_j - \bbE[\by_i], \bv \rangle\Big]^2 \notag \\
&\leq 2\widehat{\sigma}_0^2 +  \frac{2}{|\calS|}\sum_{i\in\calS} \frac{1}{|\calS|}\sum_{j\in\calS}\langle \by_j - \bbE[\by_i], \bv \rangle^2 \tag{using the Jensen's inequality} \\
&= 2\widehat{\sigma}_0^2 +  \frac{2}{|\calS|}\sum_{i\in\calS} \frac{1}{|\calS|}\sum_{j\in\calS}\left[\langle \by_j -\bbE[\by_j], \bv \rangle + \langle \bbE[\by_j] - \bbE[\by_i], \bv \rangle\right]^2 \notag \\
&\leq 2\widehat{\sigma}_0^2 +  \frac{2}{|\calS|}\sum_{i\in\calS} \frac{2}{|\calS|}\sum_{j\in\calS}\langle \by_j -\bbE[\by_j], \bv \rangle^2 + \frac{2}{|\calS|}\sum_{i\in\calS} \frac{2}{|\calS|}\sum_{j\in\calS}\langle \bbE[\by_j] - \bbE[\by_i], \bv \rangle^2 \tag{using $(a+b)^2 \leq 2a^2 + 2b^2$} \\
&\leq 2\widehat{\sigma}_0^2 +  \frac{4}{|\calS|}\sum_{j\in\calS}\langle \by_j -\bbE[\by_j], \bv \rangle^2 + \frac{4}{|\calS|}\sum_{i\in\calS} \frac{1}{|\calS|}\sum_{j\in\calS}\|\bbE[\by_j] - \bbE[\by_i]\|^2 \tag{using the Cauchy-Schwarz inequality and that $\|\bv\|\leq1$} \\
&\leq 6\widehat{\sigma}_0^2 +  \frac{4}{|\calS|}\sum_{i\in\calS} \frac{1}{|\calS|}\sum_{j\in\calS}\|\bbE[\by_j] - \bbE[\by_i]\|^2 \label{remaining_part1-robust-grad11-app}
\end{align}
\begin{claim}\label{claim:local-global-expectation-app}
For any $r,s\in\calK_{t_k}$, we have 
\begin{align}\label{local-global-expectation-bound-app}
\left\|\bbE[\by_r] - \bbE[\by_s]\right\|^2 \leq H\sum_{t=t_k}^{t_{k+1}-1}\(6\kappa^2 + 3L^2\bbE\|\bx_r^t-\bx_s^t\|^2\),
\end{align}
where expectations in $\bbE[\by_r]$ and $\bbE[\by_s]$ are taken over sampling stochastic gradients between the synchronization indices $t_k,\hdots,t_{k+1}$ by client $r$ and client $s$, respectively.
\end{claim}
\begin{proof}
Note that we can equivalently write $\bbE[\by_r]=\bbE[Y_r]$ and $\bbE[\by_s]=\bbE[Y_s]$. 
\begin{align}
\|\bbE[Y_r] - \bbE[Y_s]\|^2 &= \left\|\bbE[Y_r] - \bbE[Y_s]\right\|^2 
= \left\|\sum_{t=t_k}^{t_{k+1}-1}\Big(\bbE[Y_{r}^t] - \bbE[Y_{s}^t]\Big)\right\|^2 \notag \\
&\leq (t_{k+1}-t_k)\sum_{t=t_k}^{t_{k+1}-1}\left\|\bbE[Y_{r}^t] - \bbE[Y_{s}^t]\right\|^2 \label{local-global-expectation2-app1}
\end{align}
By definition of $Y_{s}^t$ from \eqref{distr-Yij_proof-matrix-concen-app}, we have
$Y_{s}^t \sim \text{Unif}\Big(\calF_s^{\otimes b}\big(\bx_s^{t}\big(\bx_s^{t_k},Y_{s}^{t_k},\hdots,Y_{s}^{t-1}\big)\big)\Big)$, which implies using \eqref{same_mean} that 
$\bbE[Y_{s}^t] =\bbE\left[\nabla F_s\big(\bx_s^{t}\big(\bx_s^{t_k},Y_{s}^{t_k},\hdots,Y_{s}^{t-1}\big)\big)\right]$, where on the RHS, expectation is taken over $(Y_{s}^{t_k},\hdots,Y_{s}^{t-1})$.
To make the notation less cluttered, in the following, for any $s\in\calK_{t_k}$, we write $\bx_s^{t}$ to denote $\bx_s^{t}\big(\bx_s^{t_k},Y_{s}^{t_k},\hdots,Y_{s}^{t-1}\big)$ with the understanding that expectation is always taken over the sampling of stochastic gradients between $t_k$ and $t_{k+1}$. 
With these substitutions, the $t$'th term from \eqref{local-global-expectation2-app} can be written as:
\begin{align}
\left\|\bbE[Y_{r}^t] - \bbE[Y_{s}^t]\right\|^2 &= \left\| \bbE\left[\nabla F_r(\bx_r^{t}) - \nabla F_s(\bx_s^{t})\right]\right\|^2 \notag \\
&\stackrel{\text{(a)}}{\leq} \bbE\left\| \nabla F_r\(\bx_r^{t}\) - \nabla F_s\(\bx_s^{t}\)\right\|^2 \label{local-global-expectation2-app} \\
&\stackrel{\text{(b)}}{\leq} 3\bbE\left\| \nabla F_r\(\bx_r^{t}\) - \nabla F\(\bx_r^{t}\)\right\|^2 + 3\bbE\left\| \nabla F_s\(\bx_s^{t}\) - \nabla F\(\bx_s^{t}\)\right\|^2 \notag \\
&\hspace{6cm} + 3\bbE\left\| \nabla F\(\bx_r^{t}\) - \nabla F\(\bx_s^{t}\)\right\|^2 \notag \\
&\stackrel{\text{(c)}}{\leq} 6\kappa^2 + 3L^2\bbE\|\bx_r^t-\bx_s^t\|^2. \label{local-global-expectation5-app}
\end{align}
Here, (a) and (b) both follow from the Jensen's inequality.
(c) used the gradient dissimilarity bound from \eqref{bounded_local-global} to bound the first two terms\footnote{Note that though $\bx_r^t$'s are random quantities, we can still bound $\bbE\left\|\nabla F_r(\bx_r^t) - \nabla F_s(\bx_s^t)\right\|^2\leq\kappa^2$ because the gradient dissimilarity bound \eqref{bounded_local-global} holds uniformly over the entire domain.} and $L$-Lipschitzness of $\nabla F$ to bound the last term.

Substituting the bound from \eqref{local-global-expectation5-app} back in \eqref{local-global-expectation2-app1} and using $(t_{k+1}-t_k)\leq H$ proves \Claimref{local-global-expectation-app}.
\end{proof}

Using the bound from \eqref{local-global-expectation-bound-app} in \eqref{remaining_part1-robust-grad11-app} gives
\begin{align}
\frac{1}{|\calS|}\sum_{i\in\calS}& \langle \by_i-\by_{\calS}, \bv \rangle^2 \leq 6\widehat{\sigma}_0^2 + \frac{4}{|\calS|}\sum_{i\in\calS} \frac{1}{|\calS|}\sum_{j\in\calS}H\sum_{t=t_k}^{t_{k+1}-1}\(6\kappa^2 + 3L^2\bbE\|\bx_r^t-\bx_s^t\|^2\) \notag \\
&= 6\widehat{\sigma}_0^2 + 24H^2\kappa^2 + \frac{12HL^2}{|\calS|}\sum_{i\in\calS} \frac{1}{|\calS|}\sum_{j\in\calS}\sum_{t=t_k}^{t_{k+1}-1}\bbE\|\bx_r^t-\bx_s^t\|^2 \label{remaining_part1-robust-grad13-app}
\end{align}

Now we bound the last term of \eqref{remaining_part1-robust-grad13-app}, which is the drift in local parameters at different clients in between any two synchronization indices.
\begin{lemma}\label{lem:bounded_local-global-params-app}
For any $r,s\in\calK_{t_k}$, if $\eta\leq\frac{1}{8HL}$, we have
\begin{align}\label{bounded_local-global-params-bound-app}
\sum_{t=t_k}^{t_{k+1}-1}\bbE\left\|\bx_r^{t} - \bx_s^{t}\right\|^2 \leq 7H^3\eta^2\(\frac{\sigma^2}{b}+3\kappa^2\),
\end{align}
where expectation is taken over sampling stochastic gradients at clients $r,s$ between the synchronization indices $t_k$ and $t_{k+1}$.
\end{lemma}
\begin{proof}
For any $t\in[t_k:t_{k+1}-1]$ and $r,s\in\calK_{t_k}$, define $D_{r,s}^t=\bbE\left\|\bx_r^{t} - \bx_s^{t}\right\|^2$. Note that at synchronization time $t_k$, all clients in the active set $\calK_{t_k}$ have the same parameters, i.e., $\bx_r^{t_k}=\bx^{t_k}$ for every $r\in\calK_{t_k}$.
\begin{align}
D_{r,s}^t &= \bbE\left\|\bx_r^{t} - \bx_s^{t}\right\|^2 = \bbE\left\|\left(\bx_r^{t_k} - \eta\sum_{j=t_k}^{t-1} \bg_r(\bx_r^j)\right) - \left(\bx_s^{t_k} - \eta\sum_{j=t_k}^{t-1}\bg_s(\bx_s^j)\right)\right\|^2 \nonumber \\
&= \eta^2\bbE\left\|\sum_{j=t_k}^{t-1}\left(\bg_r(\bx_r^j) - \bg_s(\bx_s^j)\right)\right\|^2 \tag{Since $\bx_r^{t_k}=\bx^{t_k}, \forall r\in\calK_{t_k}$} \\
&\leq \eta^2(t-t_k)\sum_{j=t_k}^{t-1} \bbE\left\|\bg_r(\bx_r^j) - \bg_s(\bx_s^j)\right\|^2 \nonumber \\
&\leq \eta^2H\sum_{j=t_k}^{t-1} \( 3\bbE\left\|\bg_r(\bx_r^j) - \nabla F_r(\bx_r^j)\right\|^2 + 3\bbE\left\|\bg_s(\bx_s^j) - \nabla F_s(\bx_s^j)\right\|^2 \right. \nonumber \\
&\hspace{7cm} \left. + 3\bbE\left\|\nabla F_r(\bx_r^j) - \nabla F_s(\bx_s^j)\right\|^2 \) \label{bounded_local-global-params-app-1}
\end{align}
To bound the first and the second terms we use the variance bound from \eqref{reduced_variance}.\footnote{Note that $\bx_r^j$'s are random quantities, however, since the variance bound \eqref{reduced_variance} holds uniformly over the entire domain, $\bbE\left\|\bg_r(\bx_r^j) - \nabla F_r(\bx_r^j)\right\|^2\leq\frac{\sigma^2}{b}$ holds for a random $\bx_r^j\in\R^d$.}
We can bound the third term in the same way as we bounded it in \eqref{local-global-expectation2-app} and obtained \eqref{local-global-expectation5-app}.
This gives
\begin{align}
D_{r,s}^t &\leq \eta^2H\sum_{j=t_k}^{t-1} \( \frac{6\sigma^2}{b} + 18\kappa^2 + 9L^2\bbE\|\bx_r^j - \bx_s^j\|^2\) \notag \\
&\leq \frac{6H^2\sigma^2\eta^2}{b} + 18H^2\eta^2\kappa^2 + 9L^2H\eta^2\sum_{j=t_k}^{t-1}D_{r,s}^j \tag{Since $D_{r,s}^j=\bbE\left\|\bx_r^{j} - \bx_s^{j}\right\|^2$}
\end{align}
Taking summation from $t=t_k$ to $t_{k+1}-1$ gives
\begin{align}
\sum_{t=t_k}^{t_{k+1}-1}D_{r,s}^t &\leq \sum_{t=t_k}^{t_{k+1}-1}\(\frac{6H^2\sigma^2\eta^2}{b} + 18H^2\eta^2\kappa^2 + 9L^2H\eta^2\sum_{j=t_k}^{t-1}D_{r,s}^j\) \notag \\
&\leq \frac{6H^3\sigma^2\eta^2}{b} + 18H^3\eta^2\kappa^2 + 9L^2H^2\eta^2\sum_{t=t_k}^{t_{k+1}-1}D_{r,s}^t. \notag
\end{align}
After rearranging terms, we get
\begin{align}
(1-9L^2H^2\eta^2)\sum_{t=t_k}^{t_{k+1}-1}D_{r,s}^t &\leq \frac{6H^3\sigma^2\eta^2}{b} + 18H^3\eta^2\kappa^2. \label{bounded_local-global-params-app-3}
\end{align}
If we take $\eta\leq\frac{1}{8HL}$, we get $\left(1-9\eta^2L^2H^2\right)\geq\frac{6}{7}$. 
Substituting this in the LHS of \eqref{bounded_local-global-params-app-3} yields $\sum_{t=t_k}^{t_{k+1}-1}D_{r,s}^t \leq \frac{7H^3\sigma^2\eta^2}{b} + 21H^3\eta^2\kappa^2$, which proves \Lemmaref{bounded_local-global-params-app}.
\end{proof}

Substituting the bound from \eqref{bounded_local-global-params-bound-app} for the last term in \eqref{remaining_part1-robust-grad13-app} gives
\begin{align}
\frac{1}{|\calS|}\sum_{i\in\calS} \langle \by_i-\by_{\calS}, \bv \rangle^2 &\leq 6\widehat{\sigma}_0^2 + 24H^2\kappa^2 + \frac{12HL^2}{|\calS|}\sum_{i\in\calS} \frac{1}{|\calS|}\sum_{j\in\calS}\(7H^3\eta^2\(\frac{\sigma^2}{b}+3\kappa^2\)\) \notag \\
&= 6\widehat{\sigma}_0^2 + 24H^2\kappa^2 + 84H^4L^2\eta^2\(\frac{\sigma^2}{b}+3\kappa^2\) \notag \\
&\leq 6\widehat{\sigma}_0^2 + 28H^2\kappa^2 + \frac{21H^2\sigma^2}{16b} \tag{Using $\eta\leq\frac{1}{8LH}$} \\
&\leq \frac{24H^2\sigma^2}{b\eps'}\left(1 + \frac{4d}{3K}\right) + \frac{21H^2\sigma^2}{16b} + 28H^2\kappa^2 \tag{Since $\widehat{\sigma}_0^2 = \frac{4H^2\sigma^2}{b\eps'}\left(1 + \frac{4d}{3K}\right)$} \\ 
&\leq \frac{25H^2\sigma^2}{b\eps'}\left(1 + \frac{4d}{3K}\right) + 28H^2\kappa^2. \label{remaining_part1-robust-grad15-app}
\end{align}
In the last inequality we used $\frac{21}{16}\leq \frac{1}{\eps'} \leq \frac{1}{\eps'}\(1 + \frac{4d}{3K}\)$, where the first inequality follows because $\eps'\leq\frac{1}{4}$.
Note that \eqref{remaining_part1-robust-grad15-app} holds for every unit vector $\bv\in\R^d$. 
Using this and substituting $\bg_{i,\accu}^{t_k,t_{k+1}}=\by_i,\bg_{\calS,\accu}^{t_k,t_{k+1}}=\by_{\calS}$ in \eqref{remaining_part1-robust-grad15-app}, we get 
\begin{align*}
\sup_{\bv\in\R^d:\|\bv\|=1}\frac{1}{|\calS|}\sum_{i\in\calS} \left\langle \bg_{i,\accu}^{t_k,t_{k+1}}-\bg_{\calS,\accu}^{t_k,t_{k+1}}, \bv \right\rangle^2
&\leq \frac{25H^2\sigma^2}{b\eps'}\left(1 + \frac{4d}{3K}\right) + 28H^2\kappa^2.
\end{align*}
This, in view of the alternate definition of the largest eigenvalue given in \eqref{alternate-defn_max-eigenvalue}, is equivalent to \eqref{mat-concen_gradient-estimator}, which proves the first part of \Theoremref{gradient-estimator}.

\subsection{Proof of the Second Part of \Theoremref{gradient-estimator}}\label{subsec:robust-grad-est_algo}
\begin{algorithm}[t]
   \caption{Robust Accumulated Gradient Estimation ({\sc RAGE}) \cite{Resilience_SCV18}}\label{algo:robust-grad-est}
\begin{algorithmic}[1]
   \STATE {\bf Initialize.} $c_i:=1, i\in[K]$, $\alpha:=(1-\tilde{\eps})\geq\nicefrac{3}{4}$, $\calA:=\{1,2,\hdots,K\}$; $\bG:=[\bg_1,\ \bg_2,\ \hdots,\ \bg_K]\in\R^{d\times K}$.
   \WHILE{true}
   \STATE Let $\bW^*\in\R^{|\calA|\times|\calA|}$ and $\bY^*\in\R^{d\times d}$ be the minimizer/maximizer of the saddle point problem:
   \begin{align}\label{saddle-point-opt}
   \displaystyle \max_{\substack{\bY\succeq\bzero, \\ \text{tr}(\bY)\leq 1}} \ \min_{\substack{0\leq W_{ji}\leq \frac{4-\alpha}{\alpha(2+\alpha)R}, \\ \sum_{j\in\calA}W_{ji}=1, \forall i\in\calA}} \Phi(\bW,\bY),
   \end{align}
   where the cost function $\Phi(\bW,\bY)$ is defined as   
    \begin{align}\label{cost-function_defn}
   \Phi(\bW,\bY) := \sum_{i\in\calA}c_i(\bg_i-\bG_{\calA}\bw_i)^T\bY(\bg_i-\bG_{\calA}\bw_i),
   \end{align}
      To avoid cluttered notation, we index the $|\calA|$ rows/columns of $\bW$ by the elements of $\calA$; $\bG_{\calA}$ denotes the restriction of $\bG$ to the columns in $\calA$; for $i\in\calA$, $\bw_i$ denotes the column of $\bW$ indexed by $i$.
   \STATE For $i\in\calA$, let 
   \begin{align}\label{tau-defn}
   \tau_i = (\bg_i-\bG_{\calA}\bw_i^*)^T\bY^*(\bg_i-\bG_{\calA}\bw_i^*)
   \end{align}
   \IF{$\sum_{i\in\calA}c_i\tau_i > 4R\sigma_0^2$}
   \STATE For $i\in\calA$, $c_i\leftarrow \left(1-\frac{\tau_i}{\tau_{\max}}\right)c_i$, where $\tau_{\max}=\max_{j\in\calA}\tau_{j}$.
   \STATE For all $i$ with $c_i<\frac{1}{2}$, remove $i$ from $\calA$.
   \ELSE
   \STATE Break {\bf while}-loop
   \ENDIF
   \ENDWHILE
   \STATE {\bf return} $\btg = \frac{1}{|\calA|}\sum_{i\in\calA}\bg_i$.
\end{algorithmic}
\end{algorithm}
In this section, we describe the procedure for robust mean estimation in high dimensions from \cite{Resilience_SCV18} that we use in the second part of \Theoremref{gradient-estimator} to filter-out corrupt vectors and compute an estimate of the average of uncorrupted accumulated gradients.
We refer the reader to \cite[Section 4]{DataDi_Byz-SGD_Heterogeneous20} to get an intuition on why filtering-out corrupt gradients (even when $H=1$, i.e., without local iterations) is difficult in high dimensions.

We describe the procedure in \Algorithmref{robust-grad-est} and refer the reader to \cite[Appendix E]{DataDi_Byz-SGD_Heterogeneous20} to get an intuition behind \Algorithmref{robust-grad-est} and its running-time analysis.
Though our algorithm for robust accumulated gradient estimation (RAGE) is the same as the one proposed by Steinhardt et al.~\cite{Resilience_SCV18} for high-dimensional robust mean estimation, we give it a different name, as we are applying the procedure in a much more general federated learning setting; see \Footnoteref{RAGE-vs-RME}.

For simplicity, we reorder the received gradient indices from $1,2,\hdots,K$. 
Now, the proof of the second part of \Theoremref{gradient-estimator} follows from \cite[Proposition 16]{Resilience_SCV18}, which we state below for completeness.

\begin{lemma}[Proposition 16 in \cite{Resilience_SCV18}]\label{lem:poly-time_grad-est}
Suppose we are given $K$ arbitrary vectors $\bg_1,\hdots,\bg_K\in\R^d$ with the promise that there exists a subset $\calS$ of these $K$ vectors such that  $|\calS|=(1-\tilde{\eps})K$ for some $\tilde{\eps}>0$ and $\calS$ satisfies $\lambda_{\max}\(\frac{1}{|\calS|}\sum_{i\in\calS} \(\bg_i - \bg_{\calS}\)\(\bg_i - \bg_{\calS}\)^T\) \leq \sigma_0^2$, where $\bg_{\calS}=\frac{1}{|\calS|}\sum_{i\in\calS}\bg_i$ denotes the sample mean of the vectors in $\calS$.
Then, if $\tilde{\eps}\leq\frac{1}{4}$, \Algorithmref{robust-grad-est} can find an estimate $\btg$ of $\bg_{\calS}$ in polynomial-time, such that $\|\btg-\bg_{\calS}\|\leq \calO(\sigma_0\sqrt{\tilde{\eps}})$.
\end{lemma}

Note that \Lemmaref{poly-time_grad-est} takes arbitrary vectors as inputs, which are not required to have been generated from a probability distribution.

We refer the reader to \cite[Appendix F]{DataDi_Byz-SGD_Heterogeneous20} for a comprehensive proof of \Lemmaref{poly-time_grad-est}. To analyze the running time complexity of \Algorithmref{robust-grad-est}, first note that \eqref{saddle-point-opt} can be solved by computing the singular value decomposition (SVD) of a certain $d\times K$ matrix (see \cite[Appendix F]{Resilience_SCV18} for more details), and second, that \Algorithmref{robust-grad-est} removes at least one vector in each iteration of the while loop.
So, in the worst case, \Algorithmref{robust-grad-est} requires $\calO(dK^2\min\{d,K\})$ time to execute;
see \cite[Appendix E]{DataDi_Byz-SGD_Heterogeneous20} for more details on the running time analysis of \Algorithmref{robust-grad-est}.
Note that this running time does not depend on the total number $R$ of clients (which may be in millions), and only depends on $K$, which is the number of clients selected by the server at synchronization indices. 
In federated learning, $R$ may be in millions, but $K$ is typically a small number, in $1000$'s.

This completes the proof of the second part of \Theoremref{gradient-estimator}.

\section{Convergence Proof of the Strongly-Convex Part of \Theoremref{LocalSGD_convergence}}\label{sec:convex_convergence}

At any iteration $t\in[T]$, let $\calK_t\subseteq[R]$ denote the set of clients that are active at time $t$.
Let $\bx^t:=\frac{1}{K}\sum_{r\in\calK_t}\bx_r^t$ denote the average parameter vector of the clients in the active set $\calK_t$. Note that, for any $t_i\in\I_T$, the clients in $\calK_{t_i}$ remain active at all $t\in[t_i:t_{i+1}-1]$.

In the following, we denote the decoded gradient at the server at any synchronization time $t_{i+1}$ by $\btg_\accu^{t_i,t_{i+1}}$,
which is an estimate of the average of the accumulated gradients between time $t_i$ and $t_{i+1}$ of the honest clients in $\calK_{t_i}$, as in \Theoremref{gradient-estimator}.
From \Algorithmref{Byz-Local-SGD}, we can write the parameter update rule for the global model at the synchronization indices as: 
\[\bx^{t_{i+1}} = \bx^{t_i} - \eta\btg_\accu^{t_i,t_{i+1}}.\]
Note that at any synchronization index $t_i\in\I_T$, when the server selects a subset $\calK_{t_i}$ of clients and sends the global parameter vector $\bx^{t_i}$, all clients in $\calK_{t_i}$ set their local model parameters to be equal to the global model parameters, i.e., $\bx_r^{t_i}=\bx^{t_i}$ holds for every $r\in\calK_{t_i}$.

First we derive a recurrence relation for the synchronization indices and then for non-synchronization indices.
Consider the $(i+1)$'st synchronization index $t_{i+1}\in\I_T$. We have
\begin{align}
\bx^{t_{i+1}} &= \bx^{t_{i}} - \eta\btg_\accu^{t_i,t_{i+1}} \nonumber \\
&= \bx^{t_{i}} - \eta\frac{1}{K}\sum_{r\in\calK_{t_i}}\sum_{t=t_i}^{t_{i+1}-1}\nabla F_r(\bx_r^{t}) -\eta\left(\btg_\accu^{t_i,t_{i+1}} - \frac{1}{K}\sum_{r\in\calK_{t_i}}\sum_{t=t_i}^{t_{i+1}-1}\nabla F_r(\bx_r^{t})\right) \nonumber
\end{align} 
For simplicity of notation, define $\calE\triangleq \left(\btg_\accu^{t_i,t_{i+1}} - \frac{1}{K}\sum_{r\in\calK_{t_i}}\sum_{t=t_i}^{t_{i+1}-1}\nabla F_r(\bx_r^{t})\right)$. Substituting this in the above and using $\bx^{t_{i}}=\frac{1}{K}\sum_{r\in\calK_{t_i}}\bx_r^{t_{i}}$ gives
\begin{align}
\bx^{t_{i+1}} &= \frac{1}{K}\sum_{r\in\calK_{t_i}}\bx_r^{t_{i}} - \eta\frac{1}{K}\sum_{r\in\calK_{t_i}}\sum_{t=t_i}^{t_{i+1}-1}\nabla F_r(\bx_r^{t}) -\eta\calE \nonumber \\
&= \frac{1}{K}\sum_{r\in\calK_{t_i}}\left(\bx_r^{t_{i}} - \eta\sum_{t=t_i}^{t_{i+1}-1}\nabla F_r(\bx_r^{t})\right) - \eta\calE \nonumber \\
&= \frac{1}{K}\sum_{r\in\calK_{t_i}}\left(\bx_r^{t_{i+1}-1} - \eta\nabla F_r(\bx_r^{t_{i+1}-1})\right) - \eta\calE \nonumber \\
&= \bx^{t_{i+1}-1} - \eta\frac{1}{K}\sum_{r\in\calK_{t_i}}\nabla F_r(\bx_r^{t_{i+1}-1}) - \eta\calE \nonumber \\
&= \bx^{t_{i+1}-1} - \eta\nabla F(\bx^{t_{i+1}-1}) + \eta\frac{1}{K}\sum_{r\in\calK_{t_i}}\left(\nabla F(\bx^{t_{i+1}-1}) - \nabla F_r(\bx_r^{t_{i+1}-1})\right) - \eta\calE \label{convex_local-0}
\end{align}
Subtracting $\bx^*$ from both sides gives:
\begin{align}
\bx^{t_{i+1}} - \bx^* &= \underbrace{\bx^{t_{i+1}-1} - \bx^* - \eta\nabla F(\bx^{t_{i+1}-1})}_{=:\ \bu} + \eta\underbrace{\frac{1}{K}\sum_{r\in\calK_{t_i}}\left(\nabla F(\bx^{t_{i+1}-1}) - \nabla F_r(\bx_r^{t_{i+1}-1})\right)}_{=:\ \bv} - \eta\calE \label{convex_local-01}
\end{align}
This gives $\bx^{t_{i+1}} - \bx^* = \bu + \eta(\bv - \calE)$.
Taking norm on both sides and then squaring gives
\begin{align}
\left\|\bx^{t_{i+1}} - \bx^*\right\|^2 &= \|\bu\|^2 + \eta^2\|\bv-\calE\|^2 + 2\eta\langle \bu, \bv-\calE\rangle \label{convex_local-1}
\end{align}
Now we use a simple but powerful trick on inner-products together with the inequality $2\langle \ba, \bfb\rangle \leq \|\ba\|^2 + \|\bfb\|^2$ and get:
\begin{align}
2\eta\langle \bu, \bv-\calE\rangle = 2\left\langle \sqrt{\frac{\eta\mu}{2}}\bu, \sqrt{\frac{2\eta}{\mu}}(\bv-\calE)\right\rangle \leq \frac{\eta\mu}{2}\|\bu\|^2 + \frac{2\eta}{\mu}\left\|\bv-\calE\right\|^2 \label{eq:inner-product-trick}
\end{align}
Substituting this back into \eqref{convex_local-1} gives
\begin{align}
\left\|\bx^{t_{i+1}} - \bx^*\right\|^2 &\leq \left(1+\frac{\eta\mu}{2}\right)\|\bu\|^2 + \eta\left(\eta+\frac{2}{\mu}\right)\|\bv-\calE\|^2 \nonumber \\
&\leq \left(1+\frac{\eta\mu}{2}\right)\|\bu\|^2 + 2\eta\left(\eta+\frac{2}{\mu}\right)\|\bv\|^2 + 2\eta\left(\eta+\frac{2}{\mu}\right)\|\calE\|^2 \nonumber 
\end{align}
Substituting the values of $\bu,\bv,\calE$ and taking expectation w.r.t.\ the stochastic sampling of gradients by clients in $\calK_{t_i}$ between iterations $t_i$ and $t_{i+1}$ (while conditioning on the past) gives:
\begin{align}
\bbE\left\|\bx^{t_{i+1}} - \bx^*\right\|^2 
&\leq \left(1+\frac{\mu\eta}{2}\right)\bbE\left\|\bx^{t_{i+1}-1} - \eta\nabla F(\bx^{t_{i+1}-1}) - \bx^*\right\|^2 \nonumber \\
&\hspace{1cm} + 2\eta\left(\eta+\frac{2}{\mu}\right)\bbE\left\|\frac{1}{K}\sum_{r\in\calK_{t_i}}\left(\nabla F(\bx^{t_{i+1}-1}) - \nabla F_r(\bx_r^{t_{i+1}-1})\right)\right\|^2 \nonumber \\
&\hspace{2cm} + 2\eta\left(\eta+\frac{2}{\mu}\right)\bbE\left\|\btg_\accu^{t_i,t_{i+1}} - \frac{1}{K}\sum_{r\in\calK_{t_i}}\sum_{t=t_i}^{t_{i+1}-1}\nabla F_r(\bx_r^{t})\right\|^2 \label{convex_local-2}
\end{align}
Now we bound each of the three terms on the RHS of \eqref{convex_local-2} separately in \Claimref{convex_first-term}, \Claimref{convex_second-term}, and \Claimref{convex_third-term} below. We prove these claims in \Appendixref{convex_convergence}.
\begin{claim}\label{claim:convex_first-term}
For $\eta<\frac{1}{L}$, we have
\begin{align}\label{convex_first-term-bound}
\bbE\left\|\bx^{t_{i+1}-1} - \eta\nabla F(\bx^{t_{i+1}-1}) - \bx^*\right\|^2 \leq \left(1-\mu\eta\right)\bbE\left\|\bx^{t_{i+1}-1} - \bx^*\right\|^2.
\end{align}
\end{claim}
\begin{claim}\label{claim:convex_second-term}
For $\eta\leq\frac{1}{8HL}$, we have
\begin{align}\label{convex_second-term-bound}
\bbE\left\|\frac{1}{K}\sum_{r\in\calK_{t_i}}\(\nabla F_r(\bx_r^{t_{i+1}-1}) - \nabla F(\bx^{t_{i+1}-1})\) \right\|^2 \leq 2\kappa^2 + \frac{7H}{32}\(\frac{\sigma^2}{b}+3\kappa^2\).
\end{align}
\end{claim}
\begin{claim}\label{claim:convex_third-term}
If $\eta\leq\frac{1}{8HL}$, then with probability at least $1-\exp\(-\frac{\eps'^2(1-\eps)K}{16}\)$, we have
\begin{align}\label{convex_third-term-bound}
\bbE\left\| \btg_{\emph{\accu}}^{t_i,t_{i+1}} - \frac{1}{K}\sum_{r\in\calK_{t_i}}\sum_{t=t_i}^{t_{i+1}-1}\nabla F_r(\bx_r^{t})\right\|^2 \leq 3\varUpsilon^2 + \frac{8H^2\sigma^2}{b} + 30H^2\kappa^2,
\end{align}
where $\varUpsilon^2=\calO\left(\sigma_0^2(\eps+\eps')\right)$ and $\sigma_0^2 = \frac{25H^2\sigma^2}{b\eps'}\left(1 + \frac{4d}{3K}\right) + 28H^2\kappa^2$.
\end{claim}
Substituting the bounds from \eqref{convex_first-term-bound}, \eqref{convex_second-term-bound}, \eqref{convex_third-term-bound} into \eqref{convex_local-2} and using $\left(1+\frac{\mu\eta}{2}\right)\left(1-\mu\eta\right)\leq\left(1-\frac{\mu\eta}{2}\right)$ for the first term gives
\begin{align}
\bbE\left\|\bx^{t_{i+1}} - \bx^*\right\|^2 &\leq \left(1-\frac{\mu\eta}{2}\right)\bbE\left\|\bx^{t_{i+1}-1} - \bx^*\right\|^2 + 2\eta\left(\eta+\frac{2}{\mu}\right)\(2\kappa^2 + \frac{7H}{32}\(\frac{\sigma^2}{b} + 3\kappa^2\)\) \nonumber \\
&\hspace{4cm} + 2\eta\left(\eta+\frac{2}{\mu}\right)\(3\varUpsilon^2 + \frac{8H^2\sigma^2}{b} + 30H^2\kappa^2\) \nonumber \\
&\leq \left(1-\frac{\mu\eta}{2}\right)\bbE\left\|\bx^{t_{i+1}-1} - \bx^*\right\|^2 + \frac{6\eta}{\mu}\(3\varUpsilon^2 + \frac{9H^2\sigma^2}{b} + 33H^2\kappa^2\), \label{convex_local-15}
\end{align}
where $\varUpsilon^2=\calO\left(\sigma_0^2(\eps+\eps')\right)$ and $\sigma_0^2 = \frac{25H^2\sigma^2}{b\eps'}\left(1 + \frac{4d}{3K}\right) + 28H^2\kappa^2$.
In the last inequality \eqref{convex_local-15} we used $\eta\leq\frac{1}{8LH}\leq\frac{1}{L}\leq\frac{1}{\mu}$, which implies $(\eta+\frac{2}{\mu})\leq \frac{3}{\mu}$. 
Note that \eqref{convex_local-15} holds with probability at least $1-\exp\(-\frac{\eps'^2(1-\eps)K}{16}\)$.

Note that the above recurrence in \eqref{convex_local-15} holds only at the synchronization indices $t_i\in\I_T$ for $i=1,2,3,\hdots$.
However, in order to establish a recurrence that we can use to prove convergence, we need to show a recurrence relation for all $t$. 
Now we give a recurrence at non-synchronization indices.

Take an arbitrary $t\in[T]$ and let $t_i\in\I_T$ be such that $t\in[t_i:t_{i+1}-1]$; when $H\geq2$, such $t$'s exist. Note that $\bx^t=\frac{1}{K}\sum_{r\in\calK_{t_i}}\bx_r^t$. We have
\begin{align}
&\bx^{t+1} = \bx^t - \eta\frac{1}{K}\sum_{r\in\calK_{t_i}}\bg_r(\bx_r^t) \notag \\
&\ \ = \bx^t - \eta\frac{1}{K}\sum_{r\in\calK_{t_i}}\nabla F_r(\bx_r^t) - \eta\(\frac{1}{K}\sum_{r\in\calK_{t_i}}\bg_r(\bx_r^t) - \frac{1}{K}\sum_{r\in\calK_{t_i}}\nabla F_r(\bx_r^t)\) \notag \\
&\ \ = \bx^{t} - \eta\nabla F(\bx^t) + \frac{\eta}{K}\sum_{r\in\calK_{t_i}}\left(\nabla F(\bx^t)-\nabla F_r(\bx_r^t)\right) - \frac{\eta}{K}\sum_{r\in\calK_{t_i}}\(\bg_r(\bx_r^t) - \nabla F_r(\bx_r^t)\) \label{convex_local-155}
\end{align}
Now, subtracting $\bx^*$ from both sides and following the same steps that we used to go from \eqref{convex_local-01} to \eqref{convex_local-2}, we get (in the following, expectation is taken w.r.t.\ the stochastic sampling of gradients at the $t$'th iteration while conditioning on the past):
\begin{align}
\bbE\left\|\bx^{t+1} - \bx^*\right\|^2 &\leq \(1+\frac{\mu\eta}{2}\) \bbE \left\|\bx^{t} - \bx^* - \eta\nabla F(\bx^t)\right\|^2 \nonumber \\
&\hspace{1cm} + 2\eta \(\eta+\frac{2}{\mu}\) \bbE \left\| \frac{1}{K}\sum_{r\in\calK_{t_i}} \(\nabla F(\bx^t)-\nabla F_r(\bx_r^t)\) \right\|^2 \nonumber \\
&\hspace{2cm} + 2\eta \(\eta+\frac{2}{\mu}\) \bbE \left\| \frac{1}{K}\sum_{r\in\calK_{t_i}} \(\bg_r(\bx_r^t) - \nabla F_r(\bx_r^t)\) \right\|^2 \label{convex_local-16}
\end{align}
We can bound the first and the second terms on the RHS of \eqref{convex_local-16} using \eqref{convex_first-term-bound} and \eqref{convex_second-term-bound}, respectively, as
$\bbE\left\|\bx^{t} - \eta\nabla F(\bx^{t}) - \bx^*\right\|^2 \leq \left(1-\mu\eta\right)\bbE\left\|\bx^{t} - \bx^*\right\|^2$ and
$\bbE\left\|\frac{1}{K}\sum_{r\in\calK_{t_i}}\(\nabla F(\bx^{t}) - \nabla F_r(\bx_r^{t})\)\right\|^2 \leq 2\kappa^2 + \frac{7H}{32}\(\frac{\sigma^2}{b} + 3\kappa^2\)$.
To bound the third term on the RHS of \eqref{convex_local-16}, we use the fact that variance of the sum of independent random variables is equal to the sum of the variances and that clients sample stochastic gradients $\bg_r(\bx_r^t)$ independent of each other; using this fact and \eqref{reduced_variance}, we have $\bbE \left\| \frac{1}{K}\sum_{r\in\calK_{t_i}} \(\bg_r(\bx_r^t) - \nabla F_r(\bx_r^t)\) \right\|^2 \leq \frac{\sigma^2}{bK}$. Substituting these in \eqref{convex_local-16} and using $\left(1+\frac{\mu\eta}{2}\right)\left(1-\mu\eta\right)\leq\left(1-\frac{\mu\eta}{2}\right)$ for the first term and $(\eta+\frac{2}{\mu})\leq \frac{3}{\mu}$ (which follows because $\eta\leq\frac{1}{8HL}\leq\frac{1}{L}\leq\frac{1}{\mu}$) give
\begin{align}
\bbE\left\|\bx^{t+1} - \bx^*\right\|^2 &\leq \(1-\frac{\mu\eta}{2}\)\bbE\left\|\bx^{t} - \bx^*\right\|^2 + \frac{6\eta}{\mu}\(2\kappa^2 + \frac{7H}{32}\(\frac{\sigma^2}{b} + 3\kappa^2\) + \frac{\sigma^2}{bK}\) \notag \\
&\leq \(1-\frac{\mu\eta}{2}\)\bbE\left\|\bx^{t} - \bx^*\right\|^2 + \frac{6\eta}{\mu}\(3H\kappa^2 + \frac{2H\sigma^2}{b}\) \label{convex_local-17}
\end{align}
Note that \eqref{convex_local-17} holds with probability 1.

Now we have a recurrence at the synchronization indices given in \eqref{convex_local-15} and at non-synchronization indices given in \eqref{convex_local-17}.
Let $\alpha=\left(1-\frac{\mu\eta}{2}\right)$, $\beta_1=\(3\varUpsilon^2 + \frac{9H^2\sigma^2}{b} + 33H^2\kappa^2\)$, and $\beta_2=\(3H\kappa^2 + \frac{2H\sigma^2}{b}\)$.
Substituting these and using \eqref{convex_local-15} for the synchronization indices and \eqref{convex_local-17} for the rest of the indices, we get:
\begin{align}
\bbE\left\|\bx^{T} - \bx^*\right\|^2 &\leq \alpha^T\left\|\bx^{0} - \bx^*\right\|^2 + 
\frac{6\eta}{\mu}\(\sum_{i=0}^{\nicefrac{T}{H}}\sum_{j=1}^{H-1}\alpha^{iH+j}\beta_2
+ \sum_{i=0}^{\nicefrac{T}{H}}\alpha^{iH}\beta_1 \) \label{convex_local-175} \\
&\leq \alpha^T\left\|\bx^{0} - \bx^*\right\|^2 + \frac{6\eta}{\mu}\(\sum_{i=0}^{\infty}\alpha^{i}\beta_2 + \sum_{i=0}^{\infty}\alpha^{iH}\beta_1 \) \notag \\
&= \alpha^T\left\|\bx^{0} - \bx^*\right\|^2 + \frac{6\eta}{\mu}\(\frac{1}{1-\alpha}\beta_2 + \frac{1}{1-\alpha^H}\beta_1 \) \label{convex_local-18}
\end{align}
Since $\alpha=\(1-\frac{\mu\eta}{2}\)$, we have $\alpha^H=\(1-\frac{\mu\eta}{2}\)^H \stackrel{\text{(a)}}{\leq} \exp(-\frac{\mu\eta H}{2}) \stackrel{\text{(b)}}{\leq} 1- \frac{\mu\eta H}{2} + \(\frac{\mu\eta H}{2}\)^2 \stackrel{\text{(c)}}{\leq} 1- \frac{\mu\eta H}{2} + \frac{1}{16}\frac{\mu\eta H}{2} = 1- \frac{15}{16}\frac{\mu\eta H}{2}$. 
In (a) we used the inequality $(1-\frac{1}{x})^x\leq\frac{1}{e}$ which holds for any $x>0$; 
in (b) we used $\exp(-x)\leq 1-x+x^2$ which holds for any $x\geq0$;
in (c) we used $\eta\leq\frac{1}{8HL}$ and $\mu\leq L$, which together imply $\frac{\mu\eta H}{2}\leq\frac{1}{16}$.
Substituting these in \eqref{convex_local-18} gives
\begin{align}
\bbE\left\|\bx^{T} - \bx^*\right\|^2 &\leq \(1-\frac{\mu\eta}{2}\)^T\left\|\bx^{0} - \bx^*\right\|^2 + \frac{6\eta}{\mu}\(\frac{2}{\mu\eta}\beta_2 + \frac{32}{15\mu\eta H}\beta_1 \) \notag \\
&\leq \(1-\frac{\mu\eta}{2}\)^T\left\|\bx^{0} - \bx^*\right\|^2 + \frac{6\times 32}{15\mu^2}\(\frac{15}{16}\beta_2 + \frac{1}{H}\beta_1 \) \notag \\
&\leq \(1-\frac{\mu\eta}{2}\)^T\left\|\bx^{0} - \bx^*\right\|^2 + \frac{13}{\mu^2}\(\frac{3\varUpsilon^2}{H} + \frac{11H\sigma^2}{b} + 36H\kappa^2\) \label{convex_local-19}
\end{align} 
Note that the last term on the RHS of \eqref{convex_local-19} is independent of $\eta$, which together with the dependence of $\eta$ on the first term implies that bigger the $\eta$, faster the convergence. 
Since we need $\eta\leq\frac{1}{8HL}$ for \Claimref{convex_second-term} and \Claimref{convex_third-term} to hold, we choose $\eta=\frac{1}{8HL}$.
Substituting this in \eqref{convex_local-19} yields the convergence rate \eqref{convex_convergence_rate} of \Theoremref{LocalSGD_convergence}.

\paragraph{Error probability analysis.}
Note that \eqref{convex_local-15} holds with probability at least $1-\exp\(-\frac{\eps'^2(1-\eps)K}{16}\)$ and \eqref{convex_local-17} holds with probability 1.
Since to arrive at \eqref{convex_local-175} (which leads to our final bound \eqref{convex_local-19}), we used \eqref{convex_local-15} $\frac{T}{H}$ times and \eqref{convex_local-17} $\(T-\frac{T}{H}\)$ times; as a consequence, by union bound, we have that \eqref{convex_local-19} holds with probability at least $1-\frac{T}{H}\exp\(-\frac{\eps'^2(1-\eps)K}{16}\)$, which is at least $(1-\delta)$, for any $\delta>0$, provided we run our algorithm for at most $T \leq \delta H\exp(\frac{\eps'^2(1-\eps)K}{16})$ iterations.

This concludes the proof of the strongly-convex part of \Theoremref{LocalSGD_convergence}.

\section{Convergence Proof of the Non-Convex Part of \Theoremref{LocalSGD_convergence}}\label{sec:nonconvex_convergence}
Let $\calK_{t}\subseteq[R]$ denote the subset of clients of size $|\calK_{t}|=K$ sampled at the $t$'th iteration.
For any $t\in[t_i:t_{i+1}-1]$, let $\bx^{t} = \frac{1}{K}\sum_{k\in\calK_{t_i}}\bx_k^t$ 
denote the average of the local parameters of clients in the sampling set $\calK_{t_i}$.

Similar to the proof given in \Sectionref{convex_convergence}, here also, first we derive a recurrence for the synchronization indices and then for non-synchronization indices.
For the synchronization indices $t_1,t_2,\hdots,t_{k},\hdots\in\I_T$, from \eqref{convex_local-0}, we have
\begin{align}
\bx^{t_{i+1}} &= \bx^{t_{i+1}-1} - \eta\nabla F(\bx^{t_{i+1}-1}) + \eta C \label{nonconvex_local-1}
\end{align}
where 
\begin{align}
C=\frac{1}{K}\sum_{r\in\calK_{t_i}}\left(\nabla F(\bx^{t_{i+1}-1}) - \nabla F_r(\bx_r^{t_{i+1}-1})\right) - \left(\btg_\accu^{t_i,t_{i+1}} - \frac{1}{K}\sum_{r\in\calK_{t_i}}\sum_{t=t_i}^{t_{i+1}-1}\nabla F_r(\bx_r^{t})\right). \label{nonconvex_local-2}
\end{align}
Now, using the definition of $L$-smoothness in \eqref{nonconvex_local-1}, we have 
\begin{align}
&F(\bx^{t_{i+1}}) \leq F(\bx^{t_{i+1}-1}) + \left\langle \nabla F(\bx^{t_{i+1}-1}), \bx^{t_{i+1}} - \bx^{t_{i+1}-1} \right\rangle + \frac{L}{2}\left\|\bx^{t_{i+1}} - \bx^{t_{i+1}-1}\right\|^2 \nonumber \\
&\quad= F(\bx^{t_{i+1}-1}) - \eta\left\langle \nabla F(\bx^{t_{i+1}-1}), \nabla F(\bx^{t_{i+1}-1}) - C  \right\rangle + \frac{\eta^2L}{2}\left\|\nabla F(\bx^{t_{i+1}-1}) - C \right\|^2 \nonumber \\
&\quad= F(\bx^{t_{i+1}-1}) - \eta\left\|\nabla F(\bx^{t_{i+1}-1})\right\|^2 + \eta\left\langle \nabla F(\bx^{t_{i+1}-1}), C \right\rangle + \frac{\eta^2L}{2}\left\|\nabla F(\bx^{t_{i+1}-1}) - C \right\|^2 \nonumber \\
&\quad\stackrel{\text{(a)}}{\leq} F(\bx^{t_{i+1}-1}) - \eta\left\|\nabla F(\bx^{t_{i+1}-1})\right\|^2 + \eta\left(\frac{\left\|\nabla F(\bx^{t_{i+1}-1})\right\|^2}{4} + \|C\|^2 \right) \nonumber \\
&\hspace{7cm} + \frac{\eta^2L}{2}\left\|\nabla F(\bx^{t_{i+1}-1}) - C \right\|^2 \nonumber \\
&\quad\stackrel{\text{(b)}}{\leq} F(\bx^{t_{i+1}-1}) - \frac{3\eta}{4}\left\|\nabla F(\bx^{t_{i+1}-1})\right\|^2 + \eta\|C\|^2 + \eta^2L\left(\left\|\nabla F(\bx^{t_{i+1}-1})\right\|^2 + \|C\|^2\right) \nonumber \\
&\quad= F(\bx^{t_{i+1}-1}) - \eta\left(\frac{3}{4}-\eta L\right)\left\|\nabla F(\bx^{t_{i+1}-1}) \right\|^2 + \eta\left(1+\eta L\right)\|C\|^2 \label{nonconvex_local-31}
\end{align}
In (a), we used the inequality $2\langle \ba, \bfb \rangle \leq \tau\|\ba\|^2 + \frac{1}{\tau}\|\bfb\|^2$, which holds for every $\tau>0$, and we used $\tau=\frac{1}{2}$ in (a).
In (b), we used the inequality $\|\ba+\bfb\|^2\leq2(\|\ba\|^2+\|\bfb\|^2)$.
For $\eta\leq\frac{1}{8HL}\leq\frac{1}{8L}$, we have $(\nicefrac{3}{4}-\eta L) \geq \nicefrac{1}{2}$ and $(1+\eta L) \leq \frac{9}{8}$. Substituting these in \eqref{nonconvex_local-31} and taking expectation w.r.t.\ the stochastic sampling of gradients at clients in $\calK_{i_t}$ between iterations $t_i$ and $t_{i+1}$ (while conditioning on the past) gives:
\begin{align}
\bbE[F(\bx^{t_{i+1}})] &\leq \bbE[F(\bx^{t_{i+1}-1})] - \frac{\eta}{2}\bbE\left\|\nabla F(\bx^{t_{i+1}-1}) \right\|^2 + \frac{9\eta}{8}\bbE\|C\|^2. \label{nonconvex_local-4}
\end{align}

Now we bound $\bbE\|C\|^2$.
Substituting the value of $C$ from \eqref{nonconvex_local-2} gives:
\begin{align}
\bbE\|C\|^2 &\leq 2\bbE\left\| \frac{1}{K}\sum_{r\in\calK_{t_i}}\left(\nabla F(\bx^{t_{i+1}-1}) - \nabla F_r(\bx_r^{t_{i+1}-1})\right) \right\|^2 + 2\bbE\left\|\btg_\accu^{t_i,t_{i+1}} - \frac{1}{K}\sum_{r\in\calK_{t_i}}\sum_{t=t_i}^{t_{i+1}-1}\nabla F_r(\bx_r^{t})\right\|^2 \notag \\
&\leq 2\(2\kappa^2 + \frac{7H}{32}\(\frac{\sigma^2}{b} + 3\kappa^2\) \) + 2\(3\varUpsilon^2 + \frac{8H^2\sigma^2}{b} + 30H^2\kappa^2\) \notag \\
&\leq 2\(3\varUpsilon^2 + \frac{9H^2\sigma^2}{b} + 33H^2\kappa^2\) \label{nonconvex_local-55}
\end{align}
Here, the first inequality used $\|\ba+\bfb\|^2\leq2(\|\ba\|^2+\|\bfb\|^2)$ and the second inequality used the bounds from \eqref{convex_second-term-bound} and \eqref{convex_third-term-bound}.

Substituting the bound from \eqref{nonconvex_local-55} into \eqref{nonconvex_local-4} gives
\begin{align}
\bbE[F(\bx^{t_{i+1}})] &\leq \bbE[F(\bx^{t_{i+1}-1})] - \frac{\eta}{2}\bbE\left\|\nabla F(\bx^{t_{i+1}-1}) \right\|^2 + \frac{9\eta}{4}\(3\varUpsilon^2 + \frac{9H^2\sigma^2}{b} + 33H^2\kappa^2\) \label{nonconvex_local-6}
\end{align}
where $\varUpsilon^2=\calO\left(\sigma_0^2(\eps+\eps')\right)$ and $\sigma_0^2 = \frac{25H^2\sigma^2}{b\eps'}\left(1 + \frac{4d}{3K}\right) + 28H^2\kappa^2$.
Note that \eqref{nonconvex_local-6} holds with probability at least $1-\exp\(-\frac{\eps'^2(1-\eps)K}{16}\)$.

Note that the above recurrence in \eqref{nonconvex_local-6} holds only at the synchronization indices $t_i\in\I_T$ for $i=1,2,3,\hdots$.
Now we give a recurrence at non-synchronization indices.

We have done a similar calculation in the strongly-convex part of \Theoremref{LocalSGD_convergence} in \Sectionref{convex_convergence}. 
Take an arbitrary $t\in[T]$ and let $t_i\in\I_T$ be such that $t\in[t_i:t_{i+1}-1]$; when $H\geq2$, such $t$'s exist. Note that $\bx^t=\frac{1}{K}\sum_{r\in\calK_{t_i}}\bx_r^t$.

From \eqref{convex_local-155}, we have 
$\bx^{t+1} = \bx^t - \eta\nabla F(\bx^t) + \eta D$,
where 
$$D = \frac{1}{K}\sum_{r\in\calK_{t_i}}\left(\nabla F(\bx^t)-\nabla F_r(\bx_r^t)\right) - \frac{1}{K}\sum_{r\in\calK_{t_i}}\(\bg_r(\bx_r^t) - \nabla F_r(\bx_r^t)\).$$

Using $L$-smoothness of $F$, and then performing similar algebraic manipulations that we used in order to arrive at \eqref{nonconvex_local-4}, we get:
\begin{align}
\bbE[F(\bx^{t+1})] &\leq \bbE[F(\bx^{t})] - \frac{\eta}{2}\bbE\left\|\nabla F(\bx^{t}) \right\|^2 + \frac{9\eta}{8}\bbE\|D\|^2 \label{nonconvex_local-8}
\end{align}
Now we bound $\bbE\|D\|^2$:
\begin{align}
\bbE\|D\|^2 &\leq 2\bbE\left\|\frac{1}{K}\sum_{r\in\calK_{t_i}}\left(\nabla F(\bx^t)-\nabla F_r(\bx_r^t)\right)\right\|^2 + 2\bbE\left\|\frac{1}{K}\sum_{r\in\calK_{t_i}}\(\bg_r(\bx_r^t) - \nabla F_r(\bx_r^t)\)\right\|^2 \notag \\
&\leq 2\(2\kappa^2 + \frac{7H}{32}\(\frac{\sigma^2}{b} + 3\kappa^2\) + \frac{\sigma^2}{bK}\) \notag \\
&\leq 2\(3H\kappa^2 + \frac{2H\sigma^2}{b}\) \label{nonconvex_local-85}
\end{align}
Here, the second inequality used the same bounds on both the quantities on the RHS of the first inequality that we used to go from \eqref{convex_local-16} to \eqref{convex_local-17}.

Substituting the bound on $\bbE\|D\|^2$ from \eqref{nonconvex_local-85} into \eqref{nonconvex_local-8} gives
\begin{align}
\bbE[F(\bx^{t+1})] &\leq \bbE[F(\bx^{t})] - \frac{\eta}{2}\bbE\left\|\nabla F(\bx^{t}) \right\|^2 + \frac{9\eta}{4}\(3H\kappa^2 + \frac{2H\sigma^2}{b}\) \label{nonconvex_local-9}
\end{align}
Note that \eqref{nonconvex_local-9} holds with probability 1.

Now we have a recurrence at synchronization indices given in \eqref{nonconvex_local-6} and at non-synchronization indices given in \eqref{nonconvex_local-9}.
Adding \eqref{nonconvex_local-6} and \eqref{nonconvex_local-9} from $t=0$ to $T$ (use \eqref{nonconvex_local-6} for the synchronization indices and \eqref{nonconvex_local-9} for the rest of the indices) gives:
\begin{align}
\sum_{t=0}^{T}\bbE[F(\bx^{t+1})] &\leq \sum_{t=0}^{T}\bbE[F(\bx^{t})] - \frac{\eta}{2}\sum_{t=0}^{T}\bbE\left\|\nabla F(\bx^{t}) \right\|^2 + \frac{9\eta}{4}\left[ \frac{T}{H}\(3\varUpsilon^2 + \frac{9H^2\sigma^2}{b} + 33H^2\kappa^2\)  \right. \notag \\
&\hspace{6cm} \left.+ \(T-\frac{T}{H}\)\(3H\kappa^2 + \frac{2H\sigma^2}{b}\)\right] \label{nonconvex_local-10}
\end{align}
We can simplifying the constant term in the RHS of \eqref{nonconvex_local-10} as follows:
\begin{align}
&\frac{1}{H}\(3\varUpsilon^2 + \frac{9H^2\sigma^2}{b} + 33H^2\kappa^2\) + \(1-\frac{1}{H}\)\(3H\kappa^2 + \frac{2H\sigma^2}{b}\) \notag \\
&\hspace{3cm} \leq \frac{1}{H}\(3\varUpsilon^2 + \frac{9H^2\sigma^2}{b} + 33H^2\kappa^2\) + \(3H\kappa^2 + \frac{2H\sigma^2}{b}\) \notag \\
&\hspace{3cm} \leq \frac{3\varUpsilon^2}{H} + \frac{11H\sigma^2}{b} + 36H\kappa^2 \notag
\end{align}
Substituting this in \eqref{nonconvex_local-10} and then rearranging, we get:
\begin{align}
\frac{1}{T}\sum_{t=0}^{T}\bbE\left\|\nabla F(\bx^{t}) \right\|^2 &\leq \frac{2}{\eta T}\left[\bbE[F(\bx^0)] - \bbE[F(\bx^{T+1})]\right] + \frac{9}{2}\(\frac{3\varUpsilon^2}{H} + \frac{11H\sigma^2}{b} + 36H\kappa^2\)  \label{nonconvex_local-11}
\end{align}
Note that the last term in \eqref{nonconvex_local-11} is a constant. 
So, it would be best to take the step-size $\eta$ to be as large as possible such that it satisfies $\eta\leq\frac{1}{8HL}$. We take $\eta=\frac{1}{8HL}$. Substituting this in \eqref{nonconvex_local-11} and using $F(\bx^{T+1}) \geq F(\bx^*)$ gives
\begin{align}
\frac{1}{T}\sum_{t=0}^{T}\bbE\left\|\nabla F(\bx^{t}) \right\|^2 &\leq \frac{16HL}{T}\left[\bbE[F(\bx^0)] - \bbE[F(\bx^{*})]\right] + \frac{9}{2}\(\frac{3\varUpsilon^2}{H} + \frac{11H\sigma^2}{b} + 36H\kappa^2 \), \label{nonconvex_local-12}
\end{align}
where $\varUpsilon^2=\calO\left(\sigma_0^2(\eps+\eps')\right)$ and $\sigma_0^2 = \frac{25H^2\sigma^2}{b\eps'}\left(1 + \frac{4d}{3K}\right) + 28H^2\kappa^2$. Note that \eqref{nonconvex_local-12} is the convergence rate \eqref{nonconvex_convergence_rate} in \Theoremref{LocalSGD_convergence}.

\paragraph{Error probability analysis.}
Note that \eqref{nonconvex_local-6} holds with probability at least $1-\exp\(-\frac{\eps'^2(1-\eps)K}{16}\)$ and \eqref{nonconvex_local-9} holds with probability 1.
Since to arrive at \eqref{nonconvex_local-10} (which leads to our final bound \eqref{nonconvex_local-12}), we used \eqref{nonconvex_local-6} $\frac{T}{H}$ times and \eqref{nonconvex_local-9} $\(T-\frac{T}{H}\)$ times; as a consequence, by union bound, we have that \eqref{nonconvex_local-12} holds with probability at least $1-\frac{T}{H}\exp\(-\frac{\eps'^2(1-\eps)K}{16}\)$, which is at least $(1-\delta)$, for any $\delta>0$, provided we run our algorithm for at most $T \leq \delta H\exp(\frac{\eps'^2(1-\eps)K}{16})$ iterations. 

This concludes the proof of the non-convex part of \Theoremref{LocalSGD_convergence}.

\section{Convergence Proof of \Theoremref{full-batch-GD}}\label{sec:convergence_full-batch-GD}

In this section, we focus on the case when in each local iteration clients compute {\em full-batch} gradients (instead of computing mini-batch stochastic gradients) in \Algorithmref{Byz-Local-SGD} and prove \Theoremref{full-batch-GD}.
Note that the robust accumulated gradient estimation (RAGE) result of \Theoremref{gradient-estimator} (which is for stochastic gradients) is one of the main ingredients behind the convergence analyses of \Theoremref{LocalSGD_convergence}. So, in order to prove \Theoremref{full-batch-GD}, first we need to show a RAGE result for full-batch gradients. Note that we can obtain such a result by substituting $\sigma=0$ in both the parts of \Theoremref{gradient-estimator}; however, this would give a loose bound on the approximation error in the second part. In the following, we get a tighter bound (both for RAGE and the convergence rates in \Theoremref{full-batch-GD}) by working directly with full-batch gradients. To get a RAGE result for full-batch gradients, we do a much simplified analysis than what we did before to prove \Theoremref{gradient-estimator}, and the resulting result is stated and proved below in \Theoremref{gradient-estimator_GD}.

Note that, in order to prove \Theoremref{gradient-estimator}, we showed an existence of a subset $\calS$ of honest clients (from the set $\calK$ of clients who communicate with the server)
from whom the accumulated stochastic gradients are well-concentrated, as stated in form of a matrix concentration bound \eqref{mat-concen_gradient-estimator} in the first part of \Theoremref{gradient-estimator}.
It turns out that for full-batch gradients, an analogous result can be proven directly 
(as there is no randomness due to stochastic gradients); and below we provide such a result.
Note that \Theoremref{gradient-estimator} is a probabilistic statement, where we show that with high probability, there exists a large subset $\calS\subseteq\calK$ of honest clients whose stochastic accumulated gradients are well-concentrated. In contrast, in the following result, we can deterministically take the set of {\em all} honest clients in $\calK$ to be that subset for which we can directly show the concentration.

First we setup the notation to state our main result on RAGE for full-batch gradients.
Let $\calK_t\subseteq[R]$ denote the subset of clients of size $K$ that are active at any time $t\in[0:T]$.
Let \Algorithmref{Byz-Local-SGD} generate a sequence of iterates $\{\bx_r^t:t\in[0:T],r\in\calK_t\}$ when run with a fixed step-size $\eta$ 
satisfying $\eta\leq\frac{1}{5HL}$
while minimizing a global objective function $F:\R^d\to\R$, where in any iteration, instead of sampling mini-batch stochastic gradients, every honest client takes full-batch gradients from their local datasets.
Take any two consecutive synchronization indices $t_k,t_{k+1}\in\I_T$. Note that $|t_{k+1}-t_k|\leq H$.
For an honest client $r\in\calK_{t_k}$, let $\nabla F_{r,\accu}^{t_k,t_{k+1}}:=\sum_{t=t_k}^{t_{k+1}-1}\nabla F_r(\bx_r^t)$ denote the sum of local full-batch gradients taken by client $r$ between time $t_k$ and $t_{k+1}$.
Note that at iteration $t_{k+1}$, every honest client $r\in\calK_{t_k}$ reports its local parameters $\bx_r^{t_{k+1}}$ to the server, from which server can compute $\nabla F_{r,\accu}^{t_k,t_{k+1}}$, whereas, corrupt clients may report arbitrary and adversarially chosen vectors in $\R^d$.
The goal of the server is to produce an estimate $\nabla\widehat{F}_{\accu}^{t_k,t_{k+1}}$ of the average accumulated gradients from honest clients as best as possible.

\begin{theorem}[Robust Accumulated Gradient Estimation for Full-Batch Gradient Descent]\label{thm:gradient-estimator_GD}
Suppose an $\eps$ fraction of clients who communicate with the server are corrupt. 
In the setting and notation described above, suppose
we are given $K\leq R$ accumulated full-batch gradients $\nabla\widetilde{F}_{r,\emph{\accu}}^{t_k,t_{k+1}}, r\in\calK_{t_k}$ in $\R^d$, where $\nabla\widetilde{F}_{r,\emph{\accu}}^{t_k,t_{k+1}}=\nabla F_{r,\emph{\accu}}^{t_k,t_{k+1}}$ if the $r$'th client is honest, otherwise can be arbitrary. 
Let $\calS\subseteq\calK_{t_k}$ be the subset of {\em all} honest clients in $\calK_{t_k}$ and $\nabla F_{\calS,\emph{\accu}}^{t_k,t_{k+1}}:=\frac{1}{|\calS|}\sum_{i\in\calS}\nabla F_{i,\emph{\accu}}^{t_k,t_{k+1}}$ be the sample average of uncorrupted full-batch gradients. 
If $\eps\leq\frac{1}{4}$, then with probability $1$, we can find an estimate $\nabla\btF_{\emph{\accu}}^{t_k,t_{k+1}}$ of $\nabla F_{\calS,\emph{\accu}}^{t_k,t_{k+1}}$ in polynomial-time, such that $\left\| \nabla\btF_{\emph{\accu}}^{t_k,t_{k+1}} - \nabla F_{\calS,\emph{\accu}}^{t_k,t_{k+1}} \right\| \leq \calO\left(H\kappa\sqrt{\eps}\right)$.
\end{theorem}
\begin{proof}
First we prove that 
\begin{align}\label{gradient-estimator_GD-0}
\lambda_{\max}\(\frac{1}{|\calS|}\sum_{i\in\calS} \(\nabla F_{i,\accu}^{t_k,t_{k+1}} - \nabla F_{\calS,\accu}^{t_k,t_{k+1}}\)\(\nabla F_{i,\accu}^{t_k,t_{k+1}} - \nabla F_{\calS,\accu}^{t_k,t_{k+1}}\)^T\) \leq 11H^2\kappa^2.
\end{align} 
In view of the alternate characterization the largest eigenvalue given in \eqref{alternate-defn_max-eigenvalue}, this is equivalent to showing 
\begin{align}\label{gradient-estimator_GD-1}
\sup_{\bv\in\R^d:\|\bv\|=1}\frac{1}{|\calS|}\sum_{i\in\calS} \left\langle \nabla F_{i,\accu}^{t_k,t_{k+1}} - \nabla F_{\calS,\accu}^{t_k,t_{k+1}}, \bv \right\rangle^2 \leq 11H^2\kappa^2,
\end{align} 
which we prove below. Define $F_{\accu}^{t_k,t_{k+1}} := \sum_{t=t_k}^{t_{k+1}-1}F(\bx^t)$, where $\bx^t=\frac{1}{K}\sum_{r\in\calK_{t_k}}\bx_r^t$ for any $t\in[t_k:t_{k+1}-1]$.
Take an arbitrary unit vector $\bv\in\R^d$.
\begin{align}
\frac{1}{|\calS|}\sum_{i\in\calS} &\left\langle \nabla F_{i,\accu}^{t_k,t_{k+1}} - \nabla F_{\calS,\accu}^{t_k,t_{k+1}}, \bv \right\rangle^2 \nonumber \\
&= \frac{1}{|\calS|}\sum_{i\in\calS} \left[\left\langle \nabla F_{i,\accu}^{t_k,t_{k+1}} - \nabla F_{\accu}^{t_k,t_{k+1}} + \nabla F_{\accu}^{t_k,t_{k+1}} - \nabla F_{\calS,\accu}^{t_k,t_{k+1}}, \bv \right\rangle \right]^2 \notag \\
&\leq \frac{2}{|\calS|}\sum_{i\in\calS} \left\langle \nabla F_{i,\accu}^{t_k,t_{k+1}} - \nabla F_{\accu}^{t_k,t_{k+1}}, \bv \right\rangle^2 + \frac{2}{|\calS|}\sum_{i\in\calS} \left\langle\nabla F_{\calS,\accu}^{t_k,t_{k+1}} - \nabla F_{\accu}^{t_k,t_{k+1}}, \bv \right\rangle^2 \tag{Using $\|\ba+\bfb\|^2\leq2\|\ba\|^2 + 2\|\bfb\|^2$} \\
&= \frac{2}{|\calS|}\sum_{i\in\calS} \left\langle \nabla F_{i,\accu}^{t_k,t_{k+1}} - \nabla F_{\accu}^{t_k,t_{k+1}}, \bv \right\rangle^2 + 2\left\langle\nabla F_{\calS,\accu}^{t_k,t_{k+1}} - \nabla F_{\accu}^{t_k,t_{k+1}}, \bv \right\rangle^2 \notag \\
&= \frac{2}{|\calS|}\sum_{i\in\calS} \left\langle \nabla F_{i,\accu}^{t_k,t_{k+1}} - \nabla F_{\accu}^{t_k,t_{k+1}}, \bv \right\rangle^2 + 2\left[\frac{1}{|\calS|}\sum_{i\in\calS}\left\langle\nabla F_{i,\accu}^{t_k,t_{k+1}} - \nabla F_{\accu}^{t_k,t_{k+1}}, \bv \right\rangle\right]^2 \notag \\
&\leq \frac{2}{|\calS|}\sum_{i\in\calS} \left\langle \nabla F_{i,\accu}^{t_k,t_{k+1}} - \nabla F_{\accu}^{t_k,t_{k+1}}, \bv \right\rangle^2 + \frac{2}{|\calS|}\sum_{i\in\calS}\left\langle\nabla F_{i,\accu}^{t_k,t_{k+1}} - \nabla F_{\accu}^{t_k,t_{k+1}}, \bv \right\rangle^2 \notag \\
&= \frac{4}{|\calS|}\sum_{i\in\calS} \left\langle \nabla F_{i,\accu}^{t_k,t_{k+1}} - \nabla F_{\accu}^{t_k,t_{k+1}}, \bv \right\rangle^2 \notag \\
&\leq \frac{4}{|\calS|}\sum_{i\in\calS} \left\| \nabla F_{i,\accu}^{t_k,t_{k+1}} - \nabla F_{\accu}^{t_k,t_{k+1}} \right\|^2 \tag{Using Cauchy-Schwarz inequality $\langle\bu,\bv\rangle \leq \|\bu\|\|\bv\|$ and that $\|\bv\|=1$} \\
&= \frac{4}{|\calS|}\sum_{i\in\calS} \left\| \sum_{t=t_k}^{t_{k+1}-1}\(\nabla F_i(\bx_i^t) - \nabla F(\bx^t)\) \right\|^2 \tag{Since $F_{\accu}^{t_k,t_{k+1}} = \sum_{t=t_k}^{t_{k+1}-1}F(\bx^t)$} \notag \\
&\leq \frac{4}{|\calS|}\sum_{i\in\calS} (t_{k+1}-t_k)\sum_{t=t_k}^{t_{k+1}-1}\left\| \nabla F_i(\bx_i^t) - \nabla F(\bx^t) \right\|^2 \tag{Using Jensen's inequality} \\
&\leq \frac{4H}{|\calS|}\sum_{i\in\calS} \sum_{t=t_k}^{t_{k+1}-1} \( 2\left\| \nabla F_i(\bx_i^t) - \nabla F(\bx_i^t) \right\|^2 + 2\left\| \nabla F(\bx_i^t) - \nabla F(\bx^t) \right\|^2 \) \notag \\
&\stackrel{\text{(a)}}{\leq} \frac{4H}{|\calS|}\sum_{i\in\calS} \sum_{t=t_k}^{t_{k+1}-1} \( 2\kappa^2 + 2L^2 \left\| \bx_i^t - \bx^t \right\|^2 \) \notag \\
&\leq 8H^2\kappa^2 + 8HL^2\sum_{t=t_k}^{t_{k+1}-1}\frac{1}{|\calS|}\sum_{i\in\calS}  \Big\| \bx_i^t - \frac{1}{K}\sum_{j\in\calK_{t_k}}\bx_j^t \Big\|^2 \tag{Since $\bx^t=\frac{1}{K}\sum_{j\in\calK_{t_k}}\bx_j^t$} \\
&\leq 8H^2\kappa^2 + 8HL^2\sum_{t=t_k}^{t_{k+1}-1}\frac{1}{|\calS|}\sum_{i\in\calS} \frac{1}{K}\sum_{j\in\calK_{t_k}} \left\| \bx_i^t - \bx_j^t \right\|^2 \label{gradient-estimator_GD-2}
\end{align}
The last inequality follows from the Jensen's inequality.
In (a) we used \eqref{bounded_local-global} to bound $\left\| \nabla F_i(\bx_i^t) - \nabla F(\bx_i^t) \right\|^2\leq\kappa^2$ and $L$-Lipschitz gradient property of $F$ to bound $\left\| \nabla F(\bx_i^t) - \nabla F(\bx^t) \right\| \leq L\|\bx_i^t-\bx^t\|$.

Now we bound the last term of \eqref{gradient-estimator_GD-2}.
\begin{lemma}\label{lem:bounded_local-global-params_GD}
For any $r,s\in\calK_{t_k}$, if $\eta\leq\frac{1}{5HL}$, we have
\begin{align}\label{bounded_local-global-params_GD-bound}
\sum_{t=t_k}^{t_{k+1}-1}\left\|\bx_r^{t} - \bx_s^{t}\right\|^2 \leq 7\eta^2H^3\kappa^2.
\end{align}
\end{lemma}
\begin{proof}
Note that we have shown a similar result (but, in expectation) in \Lemmaref{bounded_local-global-params-app} (on page~\pageref{lem:bounded_local-global-params-app}), which is for stochastic gradients.
We will simplify that proof to prove \Lemmaref{bounded_local-global-params_GD}, which is for full-batch deterministic gradients.

Take an arbitrary $t\in[t_k:t_{k+1}-1]$.
Following the proof of \Lemmaref{bounded_local-global-params-app} until \eqref{bounded_local-global-params-app-1} and removing the factor of $3$ inside the summation (the factor of $3$ appeared because we applied the Jensen's inequality earlier to separate the deterministic gradient term and the stochastic gradient terms) would give
\begin{align}
\left\|\bx_r^{t} - \bx_s^{t}\right\|^2 &\leq \eta^2H\sum_{j=t_k}^{t-1} \left\|\nabla F_r(\bx_r^j) - \nabla F_s(\bx_s^j)\right\|^2.
\end{align}
Following the remaining proof of \Lemmaref{bounded_local-global-params-app} from \eqref{bounded_local-global-params-app-1} until the end and substituting $\sigma=0$ gives the desired result.
\end{proof}
Substituting the bound from \eqref{bounded_local-global-params_GD-bound} into \eqref{gradient-estimator_GD-2} gives
\begin{align}
\frac{1}{|\calS|}\sum_{i\in\calS} \left\langle \nabla F_{i,\accu}^{t_k,t_{k+1}} - \nabla F_{\calS,\accu}^{t_k,t_{k+1}}, \bv \right\rangle^2 &\leq 8H^2\kappa^2 + 56H^4L^2\eta^2\kappa^2 \notag \\
&\leq 8H^2\kappa^2 + \frac{56}{25}H^2\kappa^2 \tag{Substituting $\eta\leq\frac{1}{5HL}$} \\
&\leq 11H^2\kappa^2. \label{gradient-estimator_GD-3}
\end{align}
Note that \eqref{gradient-estimator_GD-3} holds for an arbitrary unit vector $\bv\in\R^d$, implying that \eqref{gradient-estimator_GD-1} holds true. 
Since \eqref{gradient-estimator_GD-1} and \eqref{gradient-estimator_GD-0} are equivalent, we have thus shown \eqref{gradient-estimator_GD-0}.

Now apply the second part of \Theoremref{gradient-estimator} with $\calS$ being the set of all honest clients, and $\bg_{i,\accu}^{t_k,t_{k+1}}=\nabla F_{i,\accu}^{t_k,t_{k+1}}$, $\bg_{\calS,\accu}^{t_k,t_{k+1}}=\nabla F_{\calS,\accu}^{t_k,t_{k+1}}$ $\btg_{\accu}^{t_k,t_{k+1}}=\nabla \widehat{F}_{\accu}^{t_k,t_{k+1}}$, $\eps'=0$, and $\sigma_0^2=11H^2\kappa^2$. We would get that 
we can find an estimate $\nabla\btF_{\accu}^{t_k,t_{k+1}}$ of $\nabla F_{\calS,\accu}^{t_k,t_{k+1}}$ in polynomial-time, such that $\left\| \nabla\btF_{\accu}^{t_k,t_{k+1}} - \nabla F_{\calS,\accu}^{t_k,t_{k+1}}  \right\| \leq \calO\left(H\kappa\sqrt{\eps}\right)$ holds with probability 1.
\end{proof}

\Theoremref{full-batch-GD} can be proved with appropriate modifications in the proof of \Theoremref{LocalSGD_convergence}, and we prove it in \Appendixref{convergence_full-batch-GD}.

\section*{Acknowledgement}
This work was supported by the NSF grants \#1740047, \#1514731, and by the UC-NL grant LFR-18-548554.

\bibliographystyle{alpha}
\bibliography{reference}

\appendix

\section{Omitted Details from \Subsectionref{matrix-concentration}}\label{app:remaining_part1-robust-grad}
In this section, we prove \Claimref{reduced_variance_H-iters-app}.
\begin{claim*}[Restating \Claimref{reduced_variance_H-iters-app}]
For any honest client $i\in\calK_{t_k}$, we have $\bbE\|Y_i-\bbE[Y_i]\|^2\leq\frac{H^2\sigma^2}{b}$, where expectation is taken over sampling stochastic gradients by client $i$ between the synchronization indices $t_k$ and $t_{k+1}$.
\end{claim*}
\begin{proof}
Take an arbitrary honest client $i\in\calK_{t_k}$.
\begin{align*}
\bbE\|Y_i-\bbE[Y_i]\|^2 = \bbE\left\|\sum_{t=t_k}^{t_{k+1}-1}\(Y_{i}^t-\bbE[Y_{i}^t]\)\right\|^2
\stackrel{\text{(a)}}{\leq} (t_{k+1}-t_k)\sum_{t=t_k}^{t_{k+1}-1}\bbE\|Y_{i}^t-\bbE[Y_{i}^t]\|^2
\stackrel{\text{(b)}}{\leq} \frac{H^2\sigma^2}{b},
\end{align*}
where (a) follows from the Jensen's inequality; in (b) we used $(t_{k+1}-t_k) \leq H$ and that $\bbE\|Y_{i}^t-\bbE[Y_{i}^t]\|^2 \leq \frac{\sigma^2}{b}$ for all $j\in[H]$, which follows from the explanation below:
\begin{align*}
\bbE\|Y_{i}^t-\bbE[Y_{i}^t]\|^2 &= \sum_{\by_{i}^{t_k},\hdots,\by_{i}^{t-1}}\Pr\left[Y_{i}^j=\by_{i}^j, j\in[t_k:t-1]\right] \\
&\hspace{3cm} \times\bbE\left[\|Y_{i}^{t}-\bbE[Y_{i}^t]\|^2 \left. \right| Y_{i}^j=\by_{i}^j, j\in[t_k:t-1]\right] \\
&\stackrel{\text{(c)}}{\leq} \sum_{\by_{i}^{t_k},\hdots,\by_{i}^{t-1}}\Pr\left[Y_{i}^j=\by_{i}^j, j\in[t_k:t-1]\right]\cdot\frac{\sigma^2}{b} \\
&= \frac{\sigma^2}{b}.
\end{align*}
Note that $Y_{i}^t \sim \text{Unif}\Big(\calF_i^{\otimes b}\big(\bx_i^{t}\big(\bx_i^{t_k},Y_{i}^{t_k},\hdots,Y_{i}^{t-1}\big)\big)\Big)$. 
So, when we fix the values $Y_{i}^{t_k}=\by_{i}^{t_k},\hdots,Y_{i}^{t-1}=\by_{i}^{t-1}$, the parameter vector $\bx_i^{t}\big(\bx_i^{t_k},Y_{i}^{t_k}\hdots,Y_{i}^{t-1}\big)$ becomes a deterministic quantity. 
Now we can use the variance bound \eqref{reduced_variance} in order to bound $\bbE\left[\|Y_{i}^{t}-\bbE[Y_{i}^t]\|^2 \left. \right| Y_{i}^j=\by_{i}^j, j\in[t_k:t-1]\right]\leq\frac{\sigma^2}{b}$. This is what we used in (c).
\end{proof}

\section{Omitted Details from \Sectionref{convex_convergence}}\label{app:convex_convergence}
In this section, we prove \Claimref{convex_first-term}, \Claimref{convex_second-term}, and \Claimref{convex_third-term}.
\begin{claim*}[Restating \Claimref{convex_first-term}]
For $\eta<\frac{1}{L}$, we have
\begin{align*} 
\bbE\left\|\bx^{t_{i+1}-1} - \eta\nabla F(\bx^{t_{i+1}-1}) - \bx^*\right\|^2 \leq \left(1-\mu\eta\right)\bbE\left\|\bx^{t_{i+1}-1} - \bx^*\right\|^2.
\end{align*}
\end{claim*}
\begin{proof}
Expand the LHS.
\begin{align}
\bbE\left\|\bx^{t_{i+1}-1} - \bx^* - \eta\nabla F(\bx^{t_{i+1}-1})\right\|^2 &= \bbE\left\|\bx^{t_{i+1}-1} - \bx^*\right\|^2 + \eta^2\bbE\left\|\nabla F(\bx^{t_{i+1}-1})\right\|^2 \notag \\
&\hspace{2cm}  + 2\eta\bbE\left\langle \bx^* - \bx^{t_{i+1}-1}, \nabla F(\bx^{t_{i+1}-1})\right\rangle \label{convex_first-term-1}
\end{align}
We can bound the second term on the RHS using $L$-smoothness of $F$, which implies that $\|\nabla F(\bx)\|^2 \leq 2L(F(\bx)-F(\bx^*))$ holds for every $\bx\in\R^d$; see \Factref{Smooth_gradient-bound} on page \pageref{fact:Smooth_gradient-bound}.
We can bound the third term on the RHS using $\mu$-strong convexity of $F$ as follows: $\left\langle \bx^* - \bx^{t_{i+1}-1}, \nabla F(\bx^{t_{i+1}-1})\right\rangle \leq F(\bx^*) - F(\bx^{t_{i+1}-1}) - \frac{\mu}{2}\|\bx^{t_{i+1}-1} - \bx^*\|^2$. Substituting these back in \eqref{convex_first-term-1} gives:
\begin{align}
\bbE\left\|\bx^{t_{i+1}-1} - \bx^* - \eta\nabla F(\bx^{t_{i+1}-1})\right\|^2 &\leq \left(1-\mu\eta\right)\bbE\left\|\bx^{t_{i+1}-1} - \bx^*\right\|^2 \notag \\
&\hspace{2cm} - 2\eta(1-\eta L)\bbE\left(F(\bx^{t_{i+1}-1}) - F(\bx^*)\right) \label{convex_first-term-2}
\end{align}
Since $\eta<\frac{1}{L}$, we have $(1-\eta L)>0$. We also have $F(\bx^{t_{i+1}-1}) \geq F(\bx^*)$. Using these together, we can ignore the last term in the RHS of \eqref{convex_first-term-2}. This proves \Claimref{convex_first-term}.
\end{proof}
\begin{claim*}[Restating \Claimref{convex_second-term}]
For $\eta\leq\frac{1}{8HL}$, we have
\begin{align*} 
\bbE\left\|\frac{1}{K}\sum_{r\in\calK_{t_i}}\(\nabla F_r(\bx_r^{t_{i+1}-1}) - \nabla F(\bx^{t_{i+1}-1})\) \right\|^2 \leq 2\kappa^2 + \frac{7H}{32}\(\frac{\sigma^2}{b}+3\kappa^2\).
\end{align*}
\end{claim*}
\begin{proof}
By definition, we have $\bx^{t_{i+1}-1} = \frac{1}{K}\sum_{r\in\calK_{t_i}}\bx^{t_{i+1}-1}$.
\begin{align}
&\bbE\left\|\frac{1}{K}\sum_{r\in\calK_{t_i}}\(\nabla F_r(\bx_r^{t_{i+1}-1}) - \nabla F(\bx^{t_{i+1}-1})\) \right\|^2 
\leq \frac{1}{K}\sum_{r\in\calK_{t_i}}\bbE\left\|\nabla F_r(\bx_r^{t_{i+1}-1}) - \nabla F(\bx^{t_{i+1}-1}) \right\|^2 \nonumber \\
&\leq \frac{2}{K}\sum_{r\in\calK_{t_i}}\(\bbE\left\|\nabla F_r(\bx_r^{t_{i+1}-1}) - \nabla F(\bx_r^{t_{i+1}-1}) \right\|^2 + \bbE\left\|\nabla F(\bx_r^{t_{i+1}-1}) - \nabla F(\bx^{t_{i+1}-1}) \right\|^2 \) \nonumber \\
&\stackrel{\text{(a)}}{\leq} \frac{2}{K}\sum_{r\in\calK_{t_i}}\( \kappa^2 + L^2\bbE\left\|\bx_r^{t_{i+1}-1} - \bx^{t_{i+1}-1}\right\|^2\) \nonumber \\
&= 2\kappa^2 + \frac{2L^2}{K}\sum_{r\in\calK_{t_i}}\bbE\Big\|\bx_r^{t_{i+1}-1} - \frac{1}{K}\sum_{s\in\calK_{t_i}}\bx_s^{t_{i+1}-1}\Big\|^2 \nonumber \\
&\leq 2\kappa^2 + \frac{2L^2}{K}\sum_{r\in\calK_{t_i}}\frac{1}{K}\sum_{s\in\calK_{t_i}}\bbE\left\|\bx_r^{t_{i+1}-1} - \bx_s^{t_{i+1}-1}\right\|^2 \label{convex_second-term-interim1} \\
&\stackrel{\text{(b)}}{\leq} 2\kappa^2 + \frac{2L^2}{K}\sum_{r\in\calK_{t_i}}\frac{1}{K}\sum_{s\in\calK_{t_i}}\(7H^3\eta^2\(\frac{\sigma^2}{b}+3\kappa^2\)\) \nonumber  \\
&= 2\kappa^2 + 14L^2H^3\eta^2\(\frac{\sigma^2}{b}+3\kappa^2\) 
\stackrel{\text{(c)}}{\leq} 2\kappa^2 + \frac{7H}{32}\(\frac{\sigma^2}{b}+3\kappa^2\) \nonumber 
\end{align}
In (a) we used the gradient dissimilarity bound from \eqref{bounded_local-global} to bound the first term and $L$-Lipschitz gradient property of $F$ to bound the second term. 
For (b), note that we have already bounded $\sum_{t=t_i}^{t_{i+1}-1}\bbE\left\|\bx_r^{t} - \bx_s^{t}\right\|^2 \leq 7H^3\eta^2\(\frac{\sigma^2}{b}+3\kappa^2\)$ in \eqref{bounded_local-global-params-bound-app} in \Lemmaref{bounded_local-global-params-app}. Since each term in the summation is trivially bounded by the same quantity, which we used in (b) to bound $\bbE\left\|\bx_r^{t_{i+1}-1} - \bx_s^{t_{i+1}-1}\right\|^2 \leq 7H^3\eta^2\(\frac{\sigma^2}{b}+3\kappa^2\)$. In (c) we used $\eta\leq\frac{1}{8HL}$.
\end{proof}
\begin{claim*}[Restating \Claimref{convex_third-term}]
If $\eta\leq\frac{1}{8HL}$, then with probability at least $1-\exp\(-\frac{\eps'^2(1-\eps)K}{16}\)$, we have
\begin{align*} 
\bbE\left\| \btg_{\emph{\accu}}^{t_i,t_{i+1}} - \frac{1}{K}\sum_{r\in\calK_{t_i}}\sum_{t=t_i}^{t_{i+1}-1}\nabla F_r(\bx_r^{t})\right\|^2 \leq 3\varUpsilon^2 + \frac{8H^2\sigma^2}{b} + 30H^2\kappa^2,
\end{align*}
where $\varUpsilon^2=\calO\left(\sigma_0^2(\eps+\eps')\right)$ and $\sigma_0^2 = \frac{25H^2\sigma^2}{b\eps'}\left(1 + \frac{4d}{3K}\right) + 28H^2\kappa^2$.
\end{claim*}
\begin{proof}
Let $\calS\subseteq\calK_{t_i}$ denote the subset of honest clients of size $(1-(\eps+\eps'))K$, whose average accumulated gradient between time $t_i$ and $t_{i+1}$ that server approximates at time $t_{i+1}$ in \Theoremref{gradient-estimator}. Let the average accumulated gradient be denoted by $\bg_{\calS,\accu}^{t_i,t_{i+1}} = \frac{1}{|\calS|}\sum_{r\in\calS}\bg_{r,\accu}^{t_i,t_{i+1}}$, where $\bg_{r,\accu}^{t_i,t_{i+1}}=\sum_{t=t_i}^{t_{i+1}-1}\bg_{r}(\bx_r^t)$, and server approximates it by $\btg_{\accu}^{t_i,t_{i+1}}$.
Note that $\calS$ exists with probability at least $1-\exp\(-\frac{\eps'^2(1-\eps)K}{16}\)$.
To make the notation less cluttered, for every $r\in\calK_{t_i}$, define $\nabla F_r^{t_i,t_{i+1}} := \sum_{t=t_i}^{t_{i+1}-1}\nabla F_r(\bx_r^{t})$.
\begin{align}
\bbE\left\| \btg_\accu^{t_i,t_{i+1}} - \frac{1}{K}\sum_{r\in\calK_{t_i}}\nabla F_r^{t_i,t_{i+1}}\right\|^2 &\leq 3\bbE\left\| \btg_\accu^{t_i,t_{i+1}} - \frac{1}{|\calS|}\sum_{r\in\calS}\bg_{r,\accu}^{t_i,t_{i+1}} \right\|^2 \nonumber \\
&\quad 3\bbE\left\| \frac{1}{|\calS|}\sum_{r\in\calS}\bg_{r,\accu}^{t_i,t_{i+1}} - \frac{1}{|\calS|}\sum_{r\in\calS}\nabla F_r^{t_i,t_{i+1}}\right\|^2 \nonumber \\
&\qquad+ 3\bbE\left\| \frac{1}{|\calS|}\sum_{r\in\calS}\nabla F_r^{t_i,t_{i+1}} - \frac{1}{K}\sum_{s\in\calK_{t_i}}\nabla F_s^{t_i,t_{i+1}}\right\|^2 \label{convex_local-5}
\end{align}
Now we bound each term on the RHS of \eqref{convex_local-5}.

\paragraph{Bounding the first term on the RHS of \eqref{convex_local-5}.}
We can bound this using the second part of \Theoremref{gradient-estimator} as follows (note that given the first part of \Theoremref{gradient-estimator} is satisfied, the second part provides deterministic approximation guarantees, which implies that it also holds in expectation):
\begin{equation}\label{convex_local-6}
\bbE\left\| \btg_\accu^{t_i,t_{i+1}} - \frac{1}{|\calS|}\sum_{r\in\calS}\bg_{r,\accu}^{t_i,t_{i+1}}\right\|^2 \leq \varUpsilon^2,
\end{equation}
where $\varUpsilon^2=\calO\left(\sigma_0^2(\eps+\eps')\right)$ and $\sigma_0^2 = \frac{25H^2\sigma^2}{b\eps'}\left(1 + \frac{4d}{3K}\right) + 28H^2\kappa^2$.

\paragraph{Bounding the second term on the RHS of \eqref{convex_local-5}.}
We can bound this using the variance bound \eqref{reduced_variance}.
\begin{align}
\bbE\left\| \frac{1}{|\calS|}\sum_{r\in\calS}\(\bg_{r,\accu}^{t_i,t_{i+1}} - \nabla F_r^{t_i,t_{i+1}}\)\right\|^2 &=
\bbE\left\|\sum_{t=t_i}^{t_{i+1}-1}\frac{1}{|\calS|}\sum_{r\in\calS}\(\bg_{r}(\bx_r^t) - \nabla F_r(\bx_r^t)\) \right\|^2 \nonumber \\
&\stackrel{\text{(a)}}{\leq} (t_{i+1}-t_i) \sum_{t=t_i}^{t_{i+1}-1}\bbE\left\|\frac{1}{|\calS|}\sum_{r\in\calS}\(\bg_{r}(\bx_r^t) - \nabla F_r(\bx_r^t)\) \right\|^2 \nonumber \\
&\stackrel{\text{(b)}}{\leq} H \sum_{t=t_i}^{t_{i+1}-1}\frac{1}{|\calS|^2}\bbE\left\|\sum_{r\in\calS}\(\bg_{r}(\bx_r^t) - \nabla F_r(\bx_r^t)\) \right\|^2 \nonumber \\
&\stackrel{\text{(c)}}{=} H \sum_{t=t_i}^{t_{i+1}-1}\frac{1}{|\calS|^2}\sum_{r\in\calS}\bbE\left\|\bg_{r}(\bx_r^t) - \nabla F_r(\bx_r^t) \right\|^2 \nonumber \\
&\stackrel{\text{(d)}}{\leq} H \sum_{t=t_i}^{t_{i+1}-1}\frac{1}{|\calS|}\frac{\sigma^2}{b} 
\stackrel{\text{(e)}}{\leq} \frac{4H^2\sigma^2}{3bK}. \label{convex_local-65}
\end{align}
In (a) we used the Jensen's inequality. In (b) used $|t_{i+1}-t_i|\leq H$. 
In (c) we used \eqref{same_mean} (which states that $\bbE[\bg_{r}(\bx)] =  \nabla F_r(\bx)$ holds for every honest client $r\in[R]$ and $\bx\in\R^d$) together with that the stochastic gradients at different clients are sampled independently, and then 
we used the fact that the variance of independent random variables is equal to the sum of the variances. Note that $\text{Var}(\bg_{r}(\bx_r^t)) = \bbE\left\| \bg_{r}(\bx_r^t) - \nabla F_r(\bx_r^t) \right\|^2$. In (d) we used the variance bound \eqref{reduced_variance}. In (e) we used $|\calS| \geq (1-(\eps+\eps'))K\geq\frac{3K}{4}$, where the last inequality uses $(\eps+\eps')\leq\frac{1}{4}$.

\paragraph{Bounding the third term on the RHS of \eqref{convex_local-5}.}
\begin{align}
&\bbE\left\| \frac{1}{|\calS|}\sum_{r\in\calS}\nabla F_r^{t_i,t_{i+1}} - \frac{1}{K}\sum_{s\in\calK_{t_i}}\nabla F_s^{t_i,t_{i+1}} \right\|^2 = \bbE\left\| \sum_{t=t_i}^{t_{i+1}-1} \Big( \frac{1}{|\calS|}\sum_{r\in\calS}\nabla F_r(\bx_r^{t}) - \frac{1}{K}\sum_{s\in\calK_{t_i}}\nabla F_s(\bx_s^{t})\Big)\right\|^2 \nonumber \\
&\hspace{4cm}\stackrel{\text{(a)}}{\leq} H \sum_{t=t_i}^{t_{i+1}-1} \bbE\left\| \frac{1}{|\calS|}\sum_{r\in\calS}\nabla F_r(\bx_r^{t}) - \frac{1}{K}\sum_{s\in\calK_{t_i}}\nabla F_s(\bx_s^{t})\right\|^2 \label{convex_local-7}
\end{align}
In (a), first we used the Jensen's inequality and then substituted $|t_{i+1}-t_i|\leq H$. In order to bound \eqref{convex_local-7}, it suffices to bound $\bbE\left\|\frac{1}{|\calS|}\sum_{r\in\calS}\nabla F_r(\bx_r^{t}) - \frac{1}{K}\sum_{s\in\calK_{t_i}}\nabla F_s(\bx_s^{t})\right\|^2$ for every $t\in[t_i:t_{i+1}-1]$. We bound this in the following. Take an arbitrary $t\in[t_i:t_{i+1}-1]$.
\begin{align}
&\bbE\left\|\frac{1}{|\calS|}\sum_{r\in\calS}\nabla F_r(\bx_r^{t}) - \frac{1}{K}\sum_{s\in\calK_{t_i}}\nabla F_s(\bx_s^{t})\right\|^2 \leq 3\bbE\left\|\frac{1}{|\calS|}\sum_{r\in\calS}\left(\nabla F_r(\bx_r^{t}) - \nabla F(\bx_r^{t})\right)\right\|^2 \nonumber \\
&\hspace{1cm}+ 3\bbE\left\|\frac{1}{|\calS|}\sum_{r\in\calS}\nabla F(\bx_r^{t}) - \frac{1}{K}\sum_{s\in\calK_{t_i}}\nabla F(\bx_s^{t})\right\|^2 + 3\bbE\left\|\frac{1}{K}\sum_{s\in\calK_{t_i}}\(\nabla F(\bx_s^{t}) - \nabla F_s(\bx_s^{t})\) \right\|^2 \nonumber \\
&\leq \frac{3}{|\calS|}\sum_{r\in\calS}\bbE\left\|\nabla F_r(\bx_r^{t}) - \nabla F(\bx_r^{t})\right\|^2 + \frac{3}{K}\sum_{s\in\calK_{t_i}}\bbE\left\|\nabla F(\bx_s^{t}) - \nabla F_r(\bx_r^{t})\right\|^2 \nonumber \\
&\hspace{1cm} + 3\bbE\left\|\frac{1}{|\calS|}\sum_{r\in\calS}\left(\nabla F(\bx_r^{t}) - \nabla F(\bx^{t})\right) - \frac{1}{K}\sum_{s\in\calK_{t_i}}\left(\nabla F(\bx_s^{t}) - \nabla F(\bx^{t})\right)\right\|^2 \nonumber \\
&\leq 3\kappa^2 + 3\kappa^2 + 6\bbE\left\|\frac{1}{|\calS|}\sum_{r\in\calS}\nabla F(\bx_r^{t}) - \nabla F(\bx^{t})\right\|^2 + 6\bbE\left\|\frac{1}{K}\sum_{s\in\calK_{t_i}}\left(\nabla F(\bx_s^{t}) - \nabla F(\bx^{t})\right)\right\|^2 \nonumber \\
&\leq 6\kappa^2 + \frac{6}{|\calS|}\sum_{r\in\calS}\bbE\left\|\nabla F(\bx_r^{t}) - \nabla F(\bx^{t})\right\|^2 + \frac{6}{K}\sum_{s\in\calK_{t_i}}\bbE\left\|\nabla F(\bx_s^{t}) - \nabla F(\bx^{t})\right\|^2 \nonumber \\
&\leq 6\kappa^2 + \frac{6}{|\calS|}\sum_{r\in\calS}L^2\bbE\left\|\bx_r^{t} - \bx^{t}\right\|^2 + \frac{6}{K}\sum_{s\in\calK_{t_i}}L^2\bbE\left\|\bx_s^{t} - \bx^{t}\right\|^2 \nonumber \\
&= 6\kappa^2 + \frac{6L^2}{|\calS|}\sum_{r\in\calS}\bbE\Big\|\bx_r^{t} - \frac{1}{K}\sum_{s\in\calK_{t_i}}\bx_s^{t}\Big\|^2 + \frac{6L^2}{K}\sum_{r\in\calK_{t_i}}\bbE\Big\|\bx_r^{t} - \frac{1}{K}\sum_{s\in\calK_{t_i}}\bx_s^{t}\Big\|^2 \nonumber \\
&\leq 6\kappa^2 + \frac{6L^2}{|\calS|}\sum_{r\in\calS}\frac{1}{K}\sum_{s\in\calK_{t_i}}\bbE\left\|\bx_r^{t} - \bx_s^{t}\right\|^2 + \frac{6L^2}{K}\sum_{r\in\calK_{t_i}}\frac{1}{K}\sum_{s\in\calK_{t_i}}\bbE\left\|\bx_r^{t} - \bx_s^{t}\right\|^2 \nonumber 
\end{align}
Substituting this back in \eqref{convex_local-7} gives:
\begin{align}
&\bbE\left\| \frac{1}{|\calS|}\sum_{r\in\calS}\nabla F_r^{t_i,t_{i+1}} - \frac{1}{K}\sum_{s\in\calK_{t_i}}\nabla F_s^{t_i,t_{i+1}} \right\|^2 
\leq H \sum_{t=t_i}^{t_{i+1}-1} 6\kappa^2  \nonumber \\
&\hspace{1cm} + H \sum_{t=t_i}^{t_{i+1}-1} \(\frac{6L^2}{|\calS|}\sum_{r\in\calS}\frac{1}{K}\sum_{s\in\calK_{t_i}}\bbE\left\|\bx_r^{t} - \bx_s^{t}\right\|^2 + \frac{6L^2}{K}\sum_{r\in\calK_{t_i}}\frac{1}{K}\sum_{s\in\calK_{t_i}}\bbE\left\|\bx_r^{t} - \bx_s^{t}\right\|^2\) \nonumber \\ 
&\quad\stackrel{\text{(a)}}{\leq} 6H^2\kappa^2 + 6HL^2\(7H^3\eta^2\(\frac{\sigma^2}{b}+3\kappa^2\)\) + 6HL^2\(7H^3\eta^2\(\frac{\sigma^2}{b}+3\kappa^2\)\) \nonumber \\
&\quad= 6H^2\kappa^2 + 84L^2H^4\eta^2\(\frac{\sigma^2}{b}+3\kappa^2\) \nonumber \\
&\quad\stackrel{\text{(b)}}{\leq} 10H^2\kappa^2 + \frac{21H^2\sigma^2}{16b}. \label{convex_local-8}
\end{align}
In (a) we used $t_{i+1}-t_i\leq H$ and the bound $\sum_{t=t_i}^{t_{i+1}-1} \bbE\left\|\bx_r^{t} - \bx_s^{t}\right\|^2 \leq 7H^3\eta^2\(\frac{\sigma^2}{b}+3\kappa^2\)$, 
which holds when $\eta\leq\frac{1}{8HL}$; we have already shown this in \eqref{bounded_local-global-params-bound-app} in \Lemmaref{bounded_local-global-params-app}. In (b) we used $\eta\leq\frac{1}{8HL}$.

Substituting the bounds from \eqref{convex_local-6}, \eqref{convex_local-65}, \eqref{convex_local-8} into \eqref{convex_local-5} gives
\begin{align*}
\bbE\left\| \btg_\accu^{t_i,t_{i+1}} - \frac{1}{K}\sum_{r\in\calK_{t_i}}\nabla F_r^{t_i,t_{i+1}}\right\|^2 &\leq 3\varUpsilon^2 + \frac{4H^2\sigma^2}{bK} + 3\(10H^2\kappa^2 + \frac{21H^2\sigma^2}{16b}\) \\
&\leq 3\varUpsilon^2 + \frac{4H^2\sigma^2}{bK} + 30H^2\kappa^2 + \frac{4H^2\sigma^2}{b}\\
&= 3\varUpsilon^2 + \frac{8H^2\sigma^2}{b} + 30H^2\kappa^2,
\end{align*}
where $\varUpsilon^2=\calO\left(\sigma_0^2(\eps+\eps')\right)$ and $\sigma_0^2 = \frac{25H^2\sigma^2}{b\eps'}\left(1 + \frac{4d}{3K}\right) + 28H^2\kappa^2$.

This completes the proof of \Claimref{convex_third-term}.
\end{proof}

\begin{fact}\label{fact:Smooth_gradient-bound}
Let $F:\R^d\to\R$ be an $L$-smooth function with a global minimizer $\bx^*$. Then, for every $\bx\in\R^d$, we have
\begin{align*}
\| \nabla F(\bx) \|^2 \leq 2L( F(\bx) - F(\bx^*) ).
\end{align*}
\end{fact}
\begin{proof}
By definition of $L$-smoothness, we have $F(\by)  \leq F(\bx) + \langle \nabla F(\bx), \by-\bx \rangle + \frac{L}{2}\| \by-\bx \|^2$ holds for every $\bx,\by\in\R^d$. Fix an arbitrary $\bx\in\R^d$ and take infimum over $\by$ on both sides:
\begin{align*}
\inf_{\by} F(\by) & \leq \inf_{\by} \left( F(\bx) + \langle \nabla F(\bx), \by-\bx \rangle + \frac{L}{2}\| \by-\bx \|^2 \right) \\
	& \stackrel{\text{(a)}}{=}  \inf_{\bv: \|\bv\| = 1} \inf_t \left( F(\bx) + t \langle \nabla F(\bx), \bv \rangle + \frac{L t^2}{2} \right) \\
	& \stackrel{\text{(b)}}{=}  \inf_{\bv: \|\bv\| = 1} \left( F(\bx) - \frac{1}{2L} \langle \nabla F(\bx),\bv \rangle^2  \right) \\
	& \stackrel{\text{(c)}}{=}  \left( F(\bx) - \frac{1}{2L} \|\nabla F(\bx)\|^2  \right)
	\end{align*}  
The value of $t$ that minimizes the RHS of (a) is $t=-\frac{1}{L}\langle \nabla F(\bx), \bv \rangle$, this implies (b);
(c) follows from the Cauchy-Schwarz inequality: $\langle \bu, \bv \rangle \leq \| \bu \| \| \bv \|$, where equality is achieved whenever $\bu=\bv$.
Now, substituting $\inf \limits_{\by} F(\by) = F(\bx^*)$ yields the result.
\end{proof}

\section{Omitted Details from \Sectionref{convergence_full-batch-GD}}\label{app:convergence_full-batch-GD}

In this section, we prove \Theoremref{full-batch-GD}.
This can be proved along the lines of the proof of \Theoremref{LocalSGD_convergence}. 
Here we only write what changes in those proofs.
We prove the strongly-convex and non-convex parts of \Theoremref{full-batch-GD} in \Appendixref{convex-GD_convergence} and \Appendixref{nonconvex-GD_convergence}, respectively.

\subsection{Strongly-convex}\label{app:convex-GD_convergence}
Let $\calK_{t}\subseteq[R]$ denote the subset of clients of size $|\calK_{t}|=K$ that are active at the $t$'th iteration.
For any $t\in[t_i:t_{i+1}-1]$, let $\bx^{t} = \frac{1}{K}\sum_{k\in\calK_{t_i}}\bx_k^t$ denote the average of the local parameters of clients in the sampling set $\calK_{t_i}$.

Following the proof of the strongly-convex part of \Theoremref{LocalSGD_convergence} given in \Sectionref{convex_convergence} until \eqref{convex_local-2} gives
\begin{align}
\left\|\bx^{t_{i+1}} - \bx^*\right\|^2 
&\leq \left(1+\frac{\mu\eta}{2}\right)\left\|\bx^{t_{i+1}-1} - \eta\nabla F(\bx^{t_{i+1}-1}) - \bx^*\right\|^2 \nonumber \\
&\hspace{1cm} + 2\eta\left(\eta+\frac{2}{\mu}\right)\left\|\frac{1}{K}\sum_{r\in\calK_{t_i}}\left(\nabla F(\bx^{t_{i+1}-1}) - \nabla F_r(\bx_r^{t_{i+1}-1})\right)\right\|^2 \nonumber \\
&\hspace{2cm} + 2\eta\left(\eta+\frac{2}{\mu}\right)\left\|\btF_\accu^{t_i,t_{i+1}} - \frac{1}{K}\sum_{r\in\calK_{t_i}}\sum_{t=t_i}^{t_{i+1}-1}\nabla F_r(\bx_r^{t})\right\|^2 \label{convex-GD_local-2}
\end{align}
We have already bounded the first term in \Claimref{convex_first-term} (on page~\pageref{claim:convex_first-term}) by 
\begin{align}
\left\|\bx^{t_{i+1}} - \eta\nabla F(\bx^{t_{i+1}-1}) - \bx^*\right\|^2 \leq (1-\eta\mu)\left\|\bx^{t_{i+1}-1} - \bx^*\right\|^2. \label{convex-GD_local-3}
\end{align}
In order to bound the second term, we follow the proof of \Claimref{convex_second-term} 
exactly until \eqref{convex_second-term-interim1}, and then to bound $\left\|\bx_r^{t_{i+1}-1} - \bx_s^{t_{i+1}-1}\right\|^2$ for every $r,s\in\calK_{t_i}$, 
we use the bound from \eqref{bounded_local-global-params_GD-bound} in \Lemmaref{bounded_local-global-params_GD} and use $\eta\leq\frac{1}{5HL}$, which gives
\begin{align}
\left\|\frac{1}{K}\sum_{r\in\calK_{t_i}}\left(\nabla F_r(\bx^{t_{i+1}-1}) - \nabla F_r(\bx_r^{t_{i+1}-1})\right)\right\|^2 \leq 3H\kappa^2.  \label{convex-GD_local-4}
\end{align}
To bound the third term in the RHS of \eqref{convex-GD_local-2}, we can simplify the proof of \Claimref{convex_third-term}: Firstly, note that with full-batch gradients, the variance $\sigma^2$ becomes zero; secondly, as shown in \Theoremref{gradient-estimator_GD}, the robust estimation of accumulated gradients holds with probability 1. Following the proof of \Claimref{convex_third-term} with these changes and using $\eta\leq\frac{1}{5HL}$, we get 
\begin{align}
\left\|\btF_\accu^{t_i,t_{i+1}} - \frac{1}{K}\sum_{r\in\calK_{t_i}}\sum_{t=t_i}^{t_{i+1}-1}\nabla F_r(\bx_r^{t})\right\|^2  \leq 2\varUpsilon_{\GD}^2+20H^2\kappa^2,  \label{convex-GD_local-5}
\end{align}
where $\varUpsilon_{\GD}=\calO\(H\kappa\sqrt{\eps}\)$.
Substituting all these bounds from \eqref{convex-GD_local-3}-\eqref{convex-GD_local-5} into \eqref{convex-GD_local-2} and simplifying further using $\(1+\frac{\mu\eta}{2}\)\(1-\mu\eta\)\leq\(1-\frac{\mu\eta}{2}\)$ and $\left(\eta+\frac{2}{\mu}\right)\leq\frac{3}{\mu}$ gives
\begin{align}
\left\|\bx^{t_{i+1}} - \bx^*\right\|^2 &\leq \(1-\frac{\mu\eta}{2}\)\left\|\bx^{t_{i+1}-1} - \bx^*\right\|^2 
+\frac{6\eta}{\mu}\(2\varUpsilon_{\GD}^2+23H^2\kappa^2\) \label{convex-GD_local-6}
\end{align}
Note that \eqref{convex-GD_local-6} gives a recurrence at the synchronization indices. 
Now we give a recurrence at non-synchronization indices.
Take an arbitrary $t\in[T]$ and let $t_i\in\I_T$ be such that $t\in[t_i:t_{i+1}-1]$; when $H\geq2$, such $t$'s exist. 
Following the steps that we used to arrive at \eqref{convex_local-16}, we get the following (note that the last term on the RHS of \eqref{convex_local-16} is zero, as $\bg_r(\bx_r^t)=\nabla F_r(\bx_r^t)$ holds for every $r\in[R]$ and $t\in[T]$; this will also save us the factor of $2$ in the previous term as we don't have to use the Jensen's inequality to get to \eqref{convex_local-16}):
\begin{align}
\left\|\bx^{t+1} - \bx^*\right\|^2 &\leq \(1+\frac{\mu\eta}{2}\) \left\|\bx^{t} - \bx^* - \eta\nabla F(\bx^t)\right\|^2 + \eta \(\eta+\frac{2}{\mu}\) \left\| \frac{1}{K}\sum_{r\in\calK_{t}} \(\nabla F(\bx^t)-\nabla F_r(\bx_r^t)\) \right\|^2 \label{convex-GD_local-7}
\end{align}
Substituting the bounds from \eqref{convex-GD_local-3} and \eqref{convex-GD_local-4} into \eqref{convex-GD_local-7} and simplifying the coefficients as above, we get
\begin{align}
\left\|\bx^{t+1} - \bx^*\right\|^2 &\leq \(1-\frac{\mu\eta}{2}\) \left\|\bx^{t} - \bx^*\right\|^2 + \frac{3\eta}{\mu} (3H\kappa^2) \label{convex-GD_local-8}
\end{align}
Now we have a recurrence at the synchronization indices given in \eqref{convex-GD_local-6} and at non-synchronization indices given in \eqref{convex-GD_local-8}.
Let $\alpha=\left(1-\frac{\mu\eta}{2}\right)$, $\beta_1=\(2\varUpsilon_{\GD}^2+23H^2\kappa^2\)$, and $\beta_2=\(\frac{3}{2}H\kappa^2\)$.
Following the same steps that we used to arrive at \eqref{convex_local-18} gives:
\begin{align}
\left\|\bx^{T} - \bx^*\right\|^2 &\leq \alpha^T\left\|\bx^{0} - \bx^*\right\|^2 + \frac{6\eta}{\mu}\(\frac{1}{1-\alpha}\beta_2 + \frac{1}{1-\alpha^H}\beta_1 \) \label{convex-GD_local-9}
\end{align}
Since $\alpha=\(1-\frac{\mu\eta}{2}\)$, we have $\alpha^H=\(1-\frac{\mu\eta}{2}\)^H \stackrel{\text{(a)}}{\leq} \exp(-\frac{\mu\eta H}{2}) \stackrel{\text{(b)}}{\leq} 1- \frac{\mu\eta H}{2} + \(\frac{\mu\eta H}{2}\)^2 \stackrel{\text{(c)}}{\leq} 1- \frac{\mu\eta H}{2} + \frac{1}{10}\frac{\mu\eta H}{2} = 1- \frac{9}{10}\frac{\mu\eta H}{2}$. 
In (a) we used the inequality $(1-\frac{1}{x})^x\leq\frac{1}{e}$ which holds for any $x>0$; 
in (b) we used $\exp(-x)\leq 1-x+x^2$ which holds for any $x\geq0$;
in (c) we used $\eta\leq\frac{1}{5HL}$ and $\mu\leq L$, which imply $\frac{\mu\eta H}{2}\leq\frac{1}{10}$.
Substituting these in \eqref{convex-GD_local-9} gives
\begin{align}
\left\|\bx^{T} - \bx^*\right\|^2 &\leq \(1-\frac{\mu\eta}{2}\)^T\left\|\bx^{0} - \bx^*\right\|^2 + \frac{6\eta}{\mu}\(\frac{2}{\mu\eta}\beta_2 + \frac{20}{9\mu\eta H}\beta_1 \) \notag \\
&\leq \(1-\frac{\mu\eta}{2}\)^T\left\|\bx^{0} - \bx^*\right\|^2 + \frac{6\times 20}{9\mu^2}\(\frac{9}{10}\beta_2 + \frac{1}{H}\beta_1 \) \notag \\
&\leq \(1-\frac{\mu\eta}{2}\)^T\left\|\bx^{0} - \bx^*\right\|^2 + \frac{14}{\mu^2}\(\frac{2\varUpsilon_{\GD}^2}{H} + 25H\kappa^2\), \label{convex-GD_local-10}
\end{align} 
where $\varUpsilon_{\GD}=\calO\(H\kappa\sqrt{\eps}\)$.
Substituting the value of $\eta=\frac{1}{5HL}$ yields the convergence rate \eqref{convex-GD_convergence_rate} in the strongly-convex part of \Theoremref{full-batch-GD}. Note that \eqref{convex-GD_local-10} holds with probability 1.

\subsection{Non-convex}\label{app:nonconvex-GD_convergence}
Following the proof of the non-convex part of \Theoremref{LocalSGD_convergence} given in \Sectionref{nonconvex_convergence} until \eqref{nonconvex_local-4} and using $\eta\leq\frac{1}{5HL}$ gives:
\begin{align}
F(\bx^{t_{i+1}}) &\leq F(\bx^{t_{i+1}-1}) - \frac{\eta}{2}\left\|\nabla F(\bx^{t_{i+1}-1}) \right\|^2 + \frac{6\eta}{5}\|C\|^2, \label{nonconvex-GD_local-1}
\end{align}
where $C=\frac{1}{K}\sum_{r\in\calK_{t_i}}\left(\nabla F(\bx^{t_{i+1}-1}) - \nabla F_r(\bx_r^{t_{i+1}-1})\right) - \left(\btF_\accu^{t_i,t_{i+1}} - \frac{1}{K}\sum_{r\in\calK_{t_i}}\sum_{t=t_i}^{t_{i+1}-1}\nabla F_r(\bx_r^{t})\right)$.

Using the bounds from \eqref{convex-GD_local-4} and \eqref{convex-GD_local-5}, together with the Jensen's inequality, we can bound $\|C\|^2$ as follows:
\begin{align}\label{nonconvex-GD_local-3}
\|C\|^2 &\leq 2(3H\kappa^2) + 2(2\varUpsilon_{\GD}^2+20H^2\kappa^2) \leq 2(2\varUpsilon_{\GD}^2+23H^2\kappa^2)
\end{align}
Substituting the bound from \eqref{nonconvex-GD_local-3} into \eqref{nonconvex-GD_local-1} gives:
\begin{align}
F(\bx^{t_{i+1}}) &\leq F(\bx^{t_{i+1}-1}) - \frac{\eta}{2}\left\|\nabla F(\bx^{t_{i+1}-1}) \right\|^2 + \frac{12\eta}{5}\(2\varUpsilon_{\GD}^2+23H^2\kappa^2\), \label{nonconvex-GD_local-4}
\end{align}
where $\varUpsilon_{\GD}=\calO\(H\kappa\sqrt{\eps}\)$.

Note that above recurrence in \eqref{nonconvex-GD_local-4} holds only at the synchronization indices.
Now we give a recurrence at non-synchronization indices.

We have done a similar calculations in the non-convex part of \Theoremref{LocalSGD_convergence} in \Sectionref{nonconvex_convergence}.

Take an arbitrary $t\in[T]$ and let $t_i\in\I_T$ be such that $t\in[t_i:t_{i+1}-1]$; when $H\geq2$, such $t$'s exist. 
Following the same steps until \eqref{nonconvex_local-8} and using $\eta\leq\frac{1}{5HL}$ gives:
\begin{align}
F(\bx^{t+1}) &\leq F(\bx^{t}) - \frac{\eta}{2}\left\|\nabla F(\bx^{t}) \right\|^2 + \frac{6\eta}{5}\|D\|^2, \label{nonconvex_local-5}
\end{align}
where $D= \frac{1}{K}\sum_{r\in\calK_{t_i}}\left(\nabla F(\bx^t)-\nabla F_r(\bx_r^t)\right)$.

Using the bound from \eqref{convex-GD_local-4}, we have $\|D\|^2 \leq 3H\kappa^2$. Substituting this in \eqref{nonconvex_local-5} gives:
\begin{align}
F(\bx^{t+1}) &\leq F(\bx^{t}) - \frac{\eta}{2}\left\|\nabla F(\bx^{t}) \right\|^2 + \frac{6\eta}{5}(3H\kappa^2) \label{nonconvex-GD_local-6}
\end{align}
Now we have a recurrence at the synchronization indices given in \eqref{nonconvex-GD_local-4} and at non-synchronization indices given in \eqref{nonconvex-GD_local-6}.
Adding \eqref{nonconvex-GD_local-4} and \eqref{nonconvex-GD_local-6} from $t=0$ to $T$ (use \eqref{nonconvex-GD_local-4} for the synchronization indices and \eqref{nonconvex-GD_local-6} for the rest of the indices) gives:
\begin{align}
\sum_{t=0}^{T}F(\bx^{t+1}) &\leq \sum_{t=0}^{T}F(\bx^{t}) - \frac{\eta}{2}\sum_{t=0}^{T}\left\|\nabla F(\bx^{t}) \right\|^2 + \frac{12\eta}{5}\left[ \frac{T}{H}\(2\varUpsilon_{\GD}^2+23H^2\kappa^2\) + \(T-\frac{T}{H}\)\(\frac{3}{2}H\kappa^2\)\right] \label{nonconvex-GD_local-7}
\end{align}
After rearranging and simplifying the last constant terms, we get:
\begin{align}
\frac{1}{T}\sum_{t=0}^{T}\left\|\nabla F(\bx^{t}) \right\|^2 &\leq \frac{2}{\eta T}\left[F(\bx^0) - F(\bx^{T+1}) \right] + \frac{24}{5}\( \frac{2\varUpsilon_{\GD}^2}{H}+25H\kappa^2 \) \label{nonconvex-GD_local-8}
\end{align}
Note that the last term in \eqref{nonconvex-GD_local-8} is a constant. 
So, it would be best to take the step-size $\eta$ to be as large as possible such that it satisfies $\eta\leq\frac{1}{5HL}$. We take $\eta=\frac{1}{5HL}$. Substituting this in \eqref{nonconvex-GD_local-8} and using $F(\bx^{T+1}) \geq F(\bx^*)$ gives
\begin{align}
\frac{1}{T}\sum_{t=0}^{T}\left\|\nabla F(\bx^{t}) \right\|^2 &\leq \frac{10HL}{T}\left[F(\bx^0) - F(\bx^{*}) \right] + \frac{24}{5}\( \frac{2\varUpsilon_{\GD}^2}{H}+25H\kappa^2 \), \label{nonconvex-GD_local-9}
\end{align}
where $\varUpsilon_{\GD}=\calO\(H\kappa\sqrt{\eps}\)$. 
This yields the convergence rate \eqref{nonconvex-GD_convergence_rate} in the non-convex part of \Theoremref{full-batch-GD}. Note that \eqref{nonconvex-GD_local-9} holds with probability $1$.

This concludes the proof of \Theoremref{full-batch-GD}.

}

\end{document}